\documentclass[conference]{IEEEtran}
\usepackage{cite}
\usepackage{amsmath,amssymb,amsfonts}
\usepackage{algorithmic}
\usepackage{graphicx}
\usepackage{textcomp}
\usepackage{xcolor}
    
\usepackage{booktabs}
\usepackage{bm}
\usepackage{adjustbox}
\usepackage{multirow}
\usepackage{enumitem}
\usepackage{algorithm}
\usepackage{subfig}
\usepackage{url}
\usepackage{changes}
\usepackage{pdfpages}

\usepackage{tikz}
\usetikzlibrary{bayesnet}

\usepackage{amsthm}
\theoremstyle{definition}
\newtheorem{property}{Property}
\newtheorem{lemma}{Lemma}

\usepackage{todonotes}
\usepackage{xcolor}  

\newcommand{\YsubI}{Y_{I_t}}
\newcommand{\XsubI}{X_{I_t}}
\newcommand{\ThetasubI}{\Theta_{I_t}}
\newcommand{\YsubC}{Y_{C_t}}
\newcommand{\XsubC}{X_{C_t}}
\newcommand{\ThetasubC}{\Theta_{C_t}}
\newcommand{\YsubbarC}{Y_{\bar{C}_t}}
\newcommand{\XsubbarC}{X_{\bar{C}_t}}
\newcommand{\ThetasubbarC}{\Theta_{\bar{C}_t}}
\newcommand{\YsubS}{Y_{S_t}}
\newcommand{\XsubS}{X_{S_t}}
\newcommand{\ThetasubS}{\Theta_{S_t}}

\DeclareMathOperator*{\argmax}{arg\,max}
\DeclareMathOperator*{\argmin}{arg\,min}
    
\begin{document}

\title{Dynamic Model Tree for Interpretable Data Stream Learning}

\author{\IEEEauthorblockN{Johannes Haug}
\IEEEauthorblockA{\textit{University of Tuebingen} \\
Tuebingen, Germany \\
johannes-christian.haug@uni-tuebingen.de}
\and
\IEEEauthorblockN{Klaus Broelemann}
\IEEEauthorblockA{\textit{Schufa Holding AG} \\
Wiesbaden, Germany \\
klaus.broelemann@schufa.de}
\and
\IEEEauthorblockN{Gjergji Kasneci}
\IEEEauthorblockA{\textit{University of Tuebingen} \\
Tuebingen, Germany \\
gjergji.kasneci@uni-tuebingen.de}
}

\maketitle

\begin{abstract}
Data streams are ubiquitous in modern business and society. In practice, data streams may evolve over time and cannot be stored indefinitely. Effective and transparent machine learning on data streams is thus often challenging. Hoeffding Trees have emerged as a state-of-the art for online predictive modelling. They are easy to train and provide meaningful convergence guarantees under a stationary process. Yet, at the same time, Hoeffding Trees often require heuristic and costly extensions to adjust to distributional change, which may considerably impair their interpretability. In this work, we revisit Model Trees for machine learning in evolving data streams. Model Trees are able to maintain more flexible and locally robust representations of the active data concept, making them a natural fit for data stream applications. Our novel framework, called Dynamic Model Tree, satisfies desirable consistency and minimality properties. In experiments with synthetic and real-world tabular streaming data sets, we show that the proposed framework can drastically reduce the number of splits required by existing incremental decision trees. At the same time, our framework often outperforms state-of-the-art models in terms of predictive quality -- especially when concept drift is involved. Dynamic Model Trees are thus a powerful online learning framework that contributes to more lightweight and interpretable machine learning in data streams.
\end{abstract}

\begin{IEEEkeywords}
machine learning, data stream, model tree, concept drift, interpretability
\end{IEEEkeywords}

\section{Introduction}
Large-scale data streams are integral to most modern web-based applications such as online credit scoring, e-commerce or social media. Accordingly, the demand for powerful streaming machine learning models has increased. In practice, streaming or online learning models have to cope with limited hardware capacity and drifts of the data generating concept. Efficient, accurate and interpretable machine learning for evolving data streams is thus a major challenge.

Unlike traditional batch learning models, online learning models are updated incrementally. In this way, online learning models can be trained without the entire data set being available in main memory. Consequently, online learning models enable machine learning in practical applications that generate a potentially unlimited amount of data, e.g. large sensor systems or credit card transactions.

Online learning models usually have to cope with limited hardware capacity and drifts of the data generating concept. Changing customer preferences or emerging social media trends are prominent examples of such concept drift. In the worst case, concept drift may render previously learned concepts obsolete. 

Accordingly, online learning models must provide discriminative predictions and adjust to concept drift, while reducing overall resource consumption. In addition, much attention has recently been paid to the interpretability of machine learning models \cite{miller2019,rudin2019stop}. In particular, high-stakes applications and regulations (e.g. the EU General Data
Protection Regulation GDPR) require models to be interpretable. For example, if a model is used to predict the risk of recidivism or the probability of a loan default, it can be crucial to be able to describe the model in understandable terms. However, compared to other domains such as image recognition \cite{chen2018}, relatively little attention has been paid to the interpretability of machine learning models in evolving data streams. As one of the first works, we therefore briefly outline important aspects of interpretable online learning below.

\input{plot_motivation}

\subsection{On ``Interpretability'' in Evolving Data Streams}\label{sec:interpretability}
In general, we distinguish between post-hoc explainability and intrinsic interpretability \cite{Du2019}. The former concerns dedicated methods, e.g. local feature attributions \cite{kasneci2016licon,lundberg2017unified,ribeiro2016should,haug2021baselines}, that allow to explain complex (black-box) models. Conversely, we speak of intrinsic interpretability when the internal mechanics of the predictive model are inherently understandable to a human.

Interpretability has varying domain-specific definitions \cite{rudin2019stop} and cannot be measured in a standardized way \cite{Du2019}. Hence, interpretability is often represented by heuristic measures such as model size or complexity \cite{Bibal2016}. Intuitively, it is easier for humans to attribute meaning to individual model parameters when complexity is low. That is, the less complex a model is, the easier it is to interpret. For example, linear models and decision trees are typically considered highly interpretable. Specifically, the interpretability of linear models can be linked to their sparsity, i.e., the number of nonzero parameters. Similarly, the interpretability of decision trees can be quantified by the number of split nodes or the depth of the tree \cite{Moshkovitz2021}.

Since online learning models are incrementally updated, the parameters and model complexity can change between time steps. Therefore, in order to achieve interpretable online learning, we argue that it is not sufficient to deliver low complexity at each individual time step. Rather, changes in model complexity must also be comprehensible to humans. Ultimately, this requires that all updates to an interpretable online model are understandable. For example, the model should be able to answer questions like ``Why have you removed this ensemble component at time step $t$?'' or ``Why have you split this node at time step $u$?''. In this sense, online interpretability is closely related to the robustness to noise and adaptability to concept drift. For example, model adaptations could be made understandable by linking them to changes of the (approximate) data concept or, ideally, corresponding events in the real world.

Although this discussion is certainly not exhaustive, it serves as a first guide for the development of inherently interpretable online machine learning methods. Note that a more formal definition of online interpretability is beyond the scope of this paper and is left for future work.

\subsection{The State-Of-The-Art in Online Machine Learning}
Incremental decision trees have emerged as the state-of-the-art for online machine learning. The Hoeffding Tree is one of the most prominent frameworks. Hoeffding Trees use Hoeffding's inequality to decide at which time step, i.e. after how many observations, a leaf node will be split \cite{domingos2000mining,hulten2001mining,bifet2009adaptive,manapragada2018extremely}. A Hoeffding Tree comes asymptotically arbitrarily close to a hypothetical, batch-trained decision tree, given that the data generating process is stationary. Similar to batch-trained decision trees, Hoeffding Trees benefit from high efficiency and transparency. 

However, the basic Hoeffding Tree algorithm, VFDT \cite{domingos2000mining}, may grow indefinitely. This behaviour can considerably impair the performance of the VFDT and -- in the above sense -- its interpretability. In general, such infinite growth can be avoided, e.g. by extending the Hoeffding Tree with dedicated drift detection strategies \cite{bifet2009adaptive}. However, such extensions often increase the complexity and make split or prune decisions less intuitive. Moreover, Hoeffding Trees suffer practical limitations. For example, the way in which Hoeffding's inequality and heuristic purity measures are used within the framework has been repeatedly questioned \cite{rutkowski2012decision,matuszyk2013correcting,rutkowski2014cart}.

\subsection{Model Trees As An Alternative to Hoeffding Trees}
In this work, we revisit Model Trees as an alternative to Hoeffding Trees \cite{quinlan1992learning,broelemann2018gradient,potts2005incremental,ikonomovska2011learning}. Model Trees have much in common with regular decision trees, but contain simple predictive models in place of each (leaf) node. Hence, similar to an ensemble, Model Trees are a collection of weak learners that are combined in a structured way through a set of binary decisions. However, unlike Hierarchical Mixtures of Experts, Model Trees use only a single feature to split at each inner node. Accordingly, Model Trees preserve much of the simplicity of a regular decision tree.

Owing to the simple models, Model Trees are able to apply a less rigid separation of observations in the leaf nodes. In this way, Model Trees are generally more flexible regarding the active data concept than existing frameworks like the Hoeffding Tree (see Figure \ref{fig:dmt_ht}). In particular, Model Trees can represent linear relationships with only a few splits. Hence, Model Trees can usually achieve high predictive quality while using a simple and robust representation.

Replacing regular leaf nodes with simple models in an otherwise unmodified tree increases complexity. However, Model Trees often remain extremely shallow and thus interpretable, as we show in experiments. In addition, unlike Hoeffding Trees, Model Trees allow feature weights for different subgroups to be extracted directly from the simple models. In comparison to majority weighting schemes, this can be an advantage for local feature-based explanations.

\subsection{Our Contribution}
In this paper, we introduce a novel online learning framework called \textit{Dynamic Model Tree}. We show that the simple models of a Model Tree can be leveraged to define node-specific gain functions. These gain functions guarantee sensible consistency and minimality properties, which contribute to more intuitive and interpretable online learning. Compared to existing state-of-the-art methods such as Hoeffding Trees or earlier incremental Model Trees \cite{potts2005incremental,ikonomovska2011learning}, Dynamic Model Trees adapt to concept drift by design. In particular, the proposed framework does not rely on Hoeffding's inequality, heuristic purity measures or explicit concept drift detection mechanisms. Consequently, the Dynamic Model Tree eliminates some of the most fundamental weaknesses of existing online decision trees.

In summary, the contributions of this work are as follows:
\begin{itemize}
    \item We specify valuable properties related to the consistency and minimality of incremental decision trees (Section \ref{sec:properties}). Combined, these properties lead to more interpretable online learning as described above.
    \item We introduce the \textit{Dynamic Model Tree} framework (Section \ref{sec:framework}). In particular, we define generic gain functions that guarantee the above-mentioned properties and can be efficiently approximated via gradients \cite{broelemann2018gradient}.
    \item We propose an effective implementation of the Dynamic Model Tree that uses Generalized Linear Models and the negative log-likelihood loss (Section \ref{sec:implementation}). 
    \item We evaluate the Dynamic Model Tree on multiple synthetic and real-world tabular data sets with different types of concept drift (Section \ref{sec:experiments}). While maintaining high efficiency, our implementation often outperforms existing classifiers in terms of predictive quality and complexity.
\end{itemize}

\section{Related Work}\label{sec:related_work}
Incremental decision trees are a powerful class of online predictors. In the following, we briefly outline state-of-the-art algorithms based on the Hoeffding Tree, along with their limitations. Moreover, we discuss previous works on incremental Model Trees. For more information about online learning, we refer to recent surveys \cite{gama2012survey,ditzler2015learning,gomes2017survey}.

\subsection{Variations of the Hoeffding Tree}
The Very Fast Decision Tree (VFDT) is the first and basic implementation of a Hoeffding Tree \cite{domingos2000mining}. As mentioned above, it has practical limitations. In particular, the VFDT assumes that a relatively small set of past and current observations is representative of all future observations -- a misconception under realistic streaming conditions. Accordingly, the VFDT grows indefinitely and does not revisit old split decisions, which can impair its interpretability.

Most of the limitations of the basic VFDT can be overcome, e.g. by using regularization \cite{barddal2020regularized}, different probabilistic inequalities \cite{rutkowski2012decision}, gain measures \cite{matuszyk2013correcting,rutkowski2014cart} or tricks in the implementation \cite{manapragada2020emergent}. To increase the predictive performance under concept drift, the Hoeffding Tree may also be augmented with adaptation strategies like alternate tree growth \cite{hulten2001mining}, sliding windows \cite{bifet2009adaptive} or a dynamic replacement of inner nodes \cite{manapragada2018extremely}. In addition, ensembles of Hoeffding Trees, e.g. Bagging or Boosting \cite{bifet2010leveraging,montiel2020adaptive}, can increase the predictive performance of the basic models at the cost of higher overall complexity.

The Hoeffding Tree has gained popularity due to its rigorous convergence guarantees, efficiency, extensibility and accessibility via packages like MOA \cite{bifet2010moa} or scikit-multiflow \cite{montiel2018scikit}. However, the inherent limitations of the basic architecture may ultimately leave users in doubt about the reliability of the Hoeffding Tree. Hence, we argue that a different framework is needed, which offers a similar level of efficiency and extensibility, but is more flexible and interpretable in a dynamic online environment.

\subsection{Incremental Model Trees}\label{sec:incremental_MT}
Hoeffding Trees have been augmented with simple models, such as Na\"ive Bayes \cite{gama2003accurate} and Perceptrons \cite{bifet2010fast}. Such extensions often provide considerable improvements in predictive performance compared to majority-weighted leaves. Surprisingly, however, the more general family of Model Trees has received only little attention in online learning scenarios. Notable exceptions include the work by \cite{potts2005incremental}, which is aimed at stationary applications, and the FIMT-DD model \cite{ikonomovska2011learning}. FIMT-DD was introduced as a solution for online regression tasks. Similar to Hoeffding Trees, FIMT-DD applies Hoeffding's inequality to split at the inner nodes. Specifically, FIMT-DD aims to find the split that gives the largest reduction in the standard deviation of the target variable. To avoid infinite growth, FIMT-DD employs explicit concept drift detection via the Page-Hinkley test and offers various adaptation strategies. The FIMT-DD model and the Dynamic Model Tree proposed in this work have fundamental differences, which we outline in Section \ref{sec:implementation}.

\section{Preliminaries and Properties for Online Decision Tree Learning}
\label{sec:properties}
A data stream can be represented by a potentially infinite series of time steps $1,..,t,..,T$. Let $X_t \in \mathbb{R}^{n_t \times m}$ be the matrix of observations at time step $t$, where $n_t \geq 1$ is the number of observations and $m \geq 1$ is the number of features. We denote $Y_t \in \mathbb{R}^{n_t}$ the corresponding labels at time step $t$. The observations and labels are drawn from a distribution $P_t(X,Y)$, which we call the active concept at time step $t$. Concept drift is defined as a change in the active concept between two time steps, i.e. $P_{t_1}(X,Y) \neq P_{t_2}(X,Y)$.

Suppose that an incremental decision tree is parameterized by $\Theta_t$ at time step $t$. We assume that the parameters $\Theta_t$ are given by the context and therefore leave them unspecified. As described in the introduction, we generally aim for models that are discriminative and interpretable. Given our understanding of interpretable online learning and the example in Figure \ref{fig:dmt_ht}, we argue that for equal predictive power, the smaller tree should be preferred. In this context, we identify two crucial properties for training incremental decision trees.

Let $\Omega_t$ be a set of time indices up to time step $t$. Let $X_{\Omega_t}, Y_{\Omega_t}$ be sets of corresponding observations and labels and, for simplicity, let $L(\Omega_t)$ be the shorthand notation for $L(\Theta_{\Omega_t}, Y_{\Omega_t}, X_{\Omega_t})$, which denotes the estimated loss of an incremental decision tree with respect to $\Omega_t$. As before, we assume that the parameters $\Theta_{\Omega_t}$ are given by the context.

\begin{property}[Consistency with Parent Splits]\label{prop:conistency}
    Suppose we perform a split at time step $t$. Let $L_C(\Omega_t)$ be the new estimated loss after the split. An incremental decision tree algorithm is \textit{consistent with parent splits} regarding the set $\Omega_t$, if $L_C(\Omega_t) \leq L(\Omega_t)$.
\end{property}
Accordingly, we must avoid splits that would increase the estimated loss. This property primarily concerns the predictive quality of the obtained tree and is a common objective. Additionally, by choosing an adequate loss function that approximates the active data concept (which we discuss in Section \ref{sec:implementation}), Property \ref{prop:conistency} enables interpretable split decisions.

With the goal of low model complexity, i.e. high interpretability, we add a second property:

\begin{property}[Model Minimality]\label{prop:minimality}
    Suppose there exists a subtree of the incremental decision tree at time step $t$, whose loss is denoted by $L_{alt}(\Omega_t)$. An incremental decision tree algorithm preserves \textit{model minimality} regarding the set $\Omega_t$, if for $L(\Omega_t) = L_{alt}(\Omega_t)$ it retains the tree with fewer number of parameters.
\end{property}
Hence, we are bound to replace a complex tree, whenever it contains a simpler subtree that has equal predictive quality regarding $\Omega_t$. For practical purposes, this means that we have to prune or replace nodes or branches of the tree that no longer improve the estimated loss (since the number of parameters per node is usually fixed). Consequently, Property \ref{prop:minimality} also implies a mechanism to adapt to concept drift.

\section{Dynamic Model Tree}\label{sec:framework}
In this paper, we extend Model Trees to a novel framework for adaptive predictive modelling in dynamic data streams that adheres to the aforementioned properties.

A Dynamic Model Tree is constructed in a similar fashion as regular decision trees. That is, we begin with a single root node and gradually grow and prune the tree over time. Each node of a Dynamic Model Tree can be represented by a set of time indices $S_t \subseteq \{1,\dots,t\}$ corresponding to the observations that have reached the node up to time step $t$. Other than existing Model Trees, a Dynamic Model Tree maintains simple predictive models at both leaf and inner nodes (see Figure \ref{fig:dmt_components}). These models are used to identify optimal split candidates (i.e., feature-value combinations) and make predictions. Let $\XsubS, \YsubS$ be the observations and labels, and $\ThetasubS$ the parameters of the simple model at a node corresponding to the time indices in $S_t$. We aim to find the parameters that minimize a loss function $L(\cdot) \geq 0$:
\begin{align}\label{eq:objective}
    \Theta^*_{S_t} &= \argmin_{\ThetasubS} ~L(\ThetasubS, \YsubS, \XsubS) \nonumber\\ 
    &= \argmin_{\ThetasubS} ~\sum_{t \in S_t} L(\theta_t, Y_t, X_t)
\end{align}
We assume independence between time steps; a simplifying, yet common assumption in data stream learning that has been shown to work well in practice. Accordingly, we can update the parameters $\theta_t$ independently at every time step using gradient descent. The optimal parameters from the previous time step can be used as prior parameters at time step $t$. Accordingly, at every time step, we forward incoming observations to a corresponding leaf node, updating each simple model along the path. Once we have updated all relevant simple models, we attempt to grow or prune the Dynamic Model Tree. To this end, we require gain measures that account for the aspired \textit{consistency with parent splits} and \textit{model minimality}.

\begin{figure}[t]
\centering
\subfloat[Leaf Node]{
    \begin{tikzpicture}[x=0.75pt,y=0.75pt,yscale=-1,xscale=1]
    \draw    (350,100.25) -- (350,104.25) ;
    \draw [shift={(350,107.25)}, rotate = 270] [fill={rgb, 255:red, 0; green, 0; blue, 0 }  ][line width=0.08]  [draw opacity=0] (3.57,-1.72) -- (0,0) -- (3.57,1.72) -- cycle    ;
    \draw [color={rgb, 255:red, 26; green, 152; blue, 80 }  ,draw opacity=1 ] [dash pattern={on 0.84pt off 2.51pt}]  (350.4,131.55) -- (350.4,127.52) ;
    \draw [color={rgb, 255:red, 26; green, 152; blue, 80 }  ,draw opacity=1 ] [dash pattern={on 0.84pt off 2.51pt}]  (286.71,131.55) -- (286.71,135.55) ;
    \draw [shift={(286.71,138.55)}, rotate = 270] [fill={rgb, 255:red, 26; green, 152; blue, 80 }  ,fill opacity=1 ][line width=0.08]  [draw opacity=0] (3.57,-1.72) -- (0,0) -- (3.57,1.72) -- cycle    ;
    \draw [color={rgb, 255:red, 26; green, 152; blue, 80 }  ,draw opacity=1 ] [dash pattern={on 0.84pt off 2.51pt}]  (286.71,131.55) -- (414.09,131.55) ;
    \draw  [color={rgb, 255:red, 26; green, 152; blue, 80 }  ,draw opacity=1 ][fill={rgb, 255:red, 26; green, 152; blue, 80 }  ,fill opacity=0.3 ][dash pattern={on 0.84pt off 2.51pt}] (283.21,142.05) .. controls (283.21,140.12) and (284.78,138.55) .. (286.71,138.55) .. controls (288.65,138.55) and (290.21,140.12) .. (290.21,142.05) .. controls (290.21,143.98) and (288.65,145.55) .. (286.71,145.55) .. controls (284.78,145.55) and (283.21,143.98) .. (283.21,142.05) -- cycle ;
    \draw  [color={rgb, 255:red, 26; green, 152; blue, 80 }  ,draw opacity=1 ][fill={rgb, 255:red, 26; green, 152; blue, 80 }  ,fill opacity=0.3 ][dash pattern={on 0.84pt off 2.51pt}] (410.59,142.05) .. controls (410.59,140.12) and (412.15,138.55) .. (414.09,138.55) .. controls (416.02,138.55) and (417.59,140.12) .. (417.59,142.05) .. controls (417.59,143.98) and (416.02,145.55) .. (414.09,145.55) .. controls (412.15,145.55) and (410.59,143.98) .. (410.59,142.05) -- cycle ;
    \draw [color={rgb, 255:red, 26; green, 152; blue, 80 }  ,draw opacity=1 ] [dash pattern={on 0.84pt off 2.51pt}]  (414.09,131.55) -- (414.09,135.55) ;
    \draw [shift={(414.09,138.55)}, rotate = 270] [fill={rgb, 255:red, 26; green, 152; blue, 80 }  ,fill opacity=1 ][line width=0.08]  [draw opacity=0] (3.57,-1.72) -- (0,0) -- (3.57,1.72) -- cycle    ;
    \draw   (339.5,117.75) .. controls (339.5,111.95) and (344.2,107.25) .. (350,107.25) .. controls (355.8,107.25) and (360.5,111.95) .. (360.5,117.75) .. controls (360.5,123.55) and (355.8,128.25) .. (350,128.25) .. controls (344.2,128.25) and (339.5,123.55) .. (339.5,117.75) -- cycle ;
    
    \draw (332,130.9) node [anchor=north west][inner sep=0.75pt]  [font=\tiny,color={rgb, 255:red, 26; green, 152; blue, 80 }  ,opacity=1 ] [align=left] {$\displaystyle ( 3)\overset{?}{\geq } 0$};
    \draw (285.9,121.3) node [anchor=north west][inner sep=0.75pt]  [font=\tiny,color={rgb, 255:red, 26; green, 152; blue, 80 }  ,opacity=1 ] [align=left] {$\displaystyle C_{t} \subseteq S_{t}$};
    \draw (363.6,119.4) node [anchor=north west][inner sep=0.75pt]  [font=\tiny,color={rgb, 255:red, 26; green, 152; blue, 80 }  ,opacity=1 ] [align=left] {$\displaystyle \overline{C}_{t} =S_{t} \backslash C_{t}$};
    \draw (318.5,91.33) node [anchor=north west][inner sep=0.75pt]  [font=\tiny] [align=left] {$\displaystyle S_{t} \subseteq \{1,\dotsc ,t\}$};
    \draw (341.5,112) node [anchor=north west][inner sep=0.75pt]  [font=\tiny] [align=left] {$\displaystyle M^{S}$};

    \end{tikzpicture}

}
\hfill
\subfloat[Inner Node]{
    \begin{tikzpicture}[x=0.75pt,y=0.75pt,yscale=-1,xscale=1]
    \draw    (330,80.25) -- (330,84.25) ;
    \draw [shift={(330,87.25)}, rotate = 270] [fill={rgb, 255:red, 0; green, 0; blue, 0 }  ][line width=0.08]  [draw opacity=0] (3.57,-1.72) -- (0,0) -- (3.57,1.72) -- cycle    ;
    \draw [color={rgb, 255:red, 69; green, 123; blue, 157 }  ,draw opacity=1 ]   (330,111.75) -- (330,107.72) ;
    \draw [color={rgb, 255:red, 69; green, 123; blue, 157 }  ,draw opacity=1 ]   (266.31,111.75) -- (266.31,115.75) ;
    \draw [shift={(266.31,118.75)}, rotate = 270] [fill={rgb, 255:red, 69; green, 123; blue, 157 }  ,fill opacity=1 ][line width=0.08]  [draw opacity=0] (3.57,-1.72) -- (0,0) -- (3.57,1.72) -- cycle    ;
    \draw [color={rgb, 255:red, 69; green, 123; blue, 157 }  ,draw opacity=1 ]   (266.31,111.75) -- (393.69,111.75) ;
    \draw   (319.5,97.75) .. controls (319.5,91.95) and (324.2,87.25) .. (330,87.25) .. controls (335.8,87.25) and (340.5,91.95) .. (340.5,97.75) .. controls (340.5,103.55) and (335.8,108.25) .. (330,108.25) .. controls (324.2,108.25) and (319.5,103.55) .. (319.5,97.75) -- cycle ;
    \draw  [color={rgb, 255:red, 69; green, 123; blue, 157 }  ,draw opacity=1 ][fill={rgb, 255:red, 69; green, 123; blue, 157 }  ,fill opacity=1 ] (262.81,122.25) .. controls (262.81,120.32) and (264.38,118.75) .. (266.31,118.75) .. controls (268.25,118.75) and (269.81,120.32) .. (269.81,122.25) .. controls (269.81,124.18) and (268.25,125.75) .. (266.31,125.75) .. controls (264.38,125.75) and (262.81,124.18) .. (262.81,122.25) -- cycle ;
    \draw  [color={rgb, 255:red, 69; green, 123; blue, 157 }  ,draw opacity=1 ][fill={rgb, 255:red, 69; green, 123; blue, 157 }  ,fill opacity=1 ] (390.19,122.25) .. controls (390.19,120.32) and (391.75,118.75) .. (393.69,118.75) .. controls (395.62,118.75) and (397.19,120.32) .. (397.19,122.25) .. controls (397.19,124.18) and (395.62,125.75) .. (393.69,125.75) .. controls (391.75,125.75) and (390.19,124.18) .. (390.19,122.25) -- cycle ;
    \draw [color={rgb, 255:red, 69; green, 123; blue, 157 }  ,draw opacity=1 ]   (393.69,111.75) -- (393.69,115.75) ;
    \draw [shift={(393.69,118.75)}, rotate = 270] [fill={rgb, 255:red, 69; green, 123; blue, 157 }  ,fill opacity=1 ][line width=0.08]  [draw opacity=0] (3.57,-1.72) -- (0,0) -- (3.57,1.72) -- cycle    ;
    
    \draw (321.5,92) node [anchor=north west][inner sep=0.75pt]  [font=\tiny] [align=left] {$\displaystyle M^{I}$};
    \draw (294,111.5) node [anchor=north west][inner sep=0.75pt]  [font=\tiny,color={rgb, 255:red, 69; green, 123; blue, 157 }  ,opacity=1 ] [align=left] {$\displaystyle ( 4)\overset{?}{\geq } 0\lor ( 5)\overset{?}{\geq } 0$};
    \draw (265.5,101.5) node [anchor=north west][inner sep=0.75pt]  [font=\tiny,color={rgb, 255:red, 69; green, 123; blue, 157 }  ,opacity=1 ] [align=left] {$\displaystyle C_{t} \subseteq I_{t}$};
    \draw (343.5,100) node [anchor=north west][inner sep=0.75pt]  [font=\tiny,color={rgb, 255:red, 69; green, 123; blue, 157 }  ,opacity=1 ] [align=left] {$\displaystyle \overline{C}_{t} =I_{t} \backslash C_{t}$};
    \draw (298.5,71.33) node [anchor=north west][inner sep=0.75pt]  [font=\tiny] [align=left] {$\displaystyle I_{t} \subseteq \{1,\dotsc ,t\}$};

    \end{tikzpicture}

}
\caption{\textbf{DMT Nodes:} Both inner and leaf nodes of a Dynamic Model Tree contain simple models $M$ that are incrementally trained during a subset of time steps $S_t$ and $I_t$, respectively. At every time step $t$, we check at the leaf nodes whether there is a new split candidate with positive gain \eqref{eq:gain_leaf} (green, see also Algorithm \ref{alg:pseudo_code}). Similarly, we check at the inner nodes whether the gains \eqref{eq:gain_inner} or \eqref{eq:gain_prune} are positive, i.e., whether we must replace the current split (blue) and thus prune the old branch.}
\label{fig:dmt_components}
\end{figure}
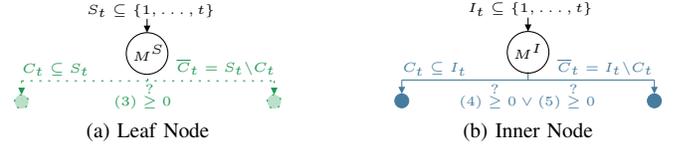

\subsection{Loss-Based Gain Functions}
Typically, decision tree algorithms aim for maximum node purity with respect to the target variable. For this purpose, split decisions are usually based on heuristic purity measures such as the Information Gain or the Gini index. However, the simple models of a Dynamic Model Tree offer a fundamental advantage in terms of the proposed properties. Instead of relying on heuristic measures, we may directly select the split candidate that reduces the overall loss of our tree. Consequently, any update of the model complexity can be directly linked to a change in the loss, providing better interpretability as described in Section \ref{sec:interpretability}.

Suppose we are at a leaf node of the tree. Let $S_t$ be the corresponding set of time indices observed at this leaf node. Our goal now is to find a new split candidate, i.e., a feature-value pair, to further split the observations. We can represent each split candidate by a set of time indices that would have been passed to the left child $C_t \subseteq S_t$ and the right child $\bar{C}_t = S_t \backslash C_t$. For the sake of illustration, we assume binary splits. However, our exposition can readily be extended to non-binary trees. Our goal is to select the split candidate that maximizes the improvement of the current loss:
\begin{align}
    C^*_t &= \argmax_{C_t}~ G_{S_t, C_t} \text{, with}\\
    G_{S_t, C_t} &= L(\ThetasubS, \YsubS, \XsubS) \nonumber \\ 
    &- L(\ThetasubC, \YsubC, \XsubC) - L(\ThetasubbarC, \YsubbarC, \XsubbarC) \label{eq:gain_leaf}
\end{align}
With \eqref{eq:gain_leaf}, the proof of \textit{consistency to parent splits} is almost trivial:
\begin{lemma}
    Every new split with a gain $G_{S_t, C_t} \geq 0$ due to \eqref{eq:gain_leaf} implies \textit{consistency with parent splits} (Property \ref{prop:conistency}).
\end{lemma}
\begin{proof}
    The loss of a Dynamic Model Tree at time step $t$ corresponds to the sum of losses at each leaf node, i.e. $L(\Omega_t) = \sum_{J_t \subseteq \Omega_t} L(\Theta_{J_t}, Y_{J_t}, X_{J_t})$, where every set $J_t$ represents a leaf node. Suppose there exists a leaf node $S_t$, such that $G_{S_t, C_t} \geq 0$ for some split candidate $C_t$. A split on $C_t$ corresponds to a new loss $L_C(\Omega_t) = L(\Omega_t) - G_{S_t, C_t}$, which implies $L_C(\Omega_t) \leq L(\Omega_t)$.
\end{proof}
To satisfy \textit{model minimality} (Property \ref{prop:minimality}), we also need to evaluate existing splits of the Dynamic Model Tree. Specifically, we may replace an existing inner node with either a new split candidate or a leaf node. In both cases, we would prune the old branch (subtree). Suppose there is a subtree whose root corresponds to an inner node of the original tree. As before, we represent this inner node by a set of time indices $I_t$. Likewise, each leaf node of the subtree is represented by a set $J_t$, such that the union of all $J_t$ is equal to $I_t$. We then try to find an alternate split candidate (represented by $C_t \subseteq I_t$, $\bar{C}_t = I_t \backslash C_t$), i.e. a substitute for the inner node $I_t$, which offers an improvement in terms of the loss:
\begin{align}
    G_{I_t, C_t} &= \sum_{J_t \subseteq I_t} L(\Theta_{J_t}, Y_{J_t}, X_{J_t}) \nonumber\\
    &- L(\ThetasubC, \YsubC, \XsubC) - L(\ThetasubbarC, \YsubbarC, \XsubbarC)\label{eq:gain_inner}
\end{align}
If the gain \eqref{eq:gain_inner} is positive, we can prune the old subtree and add a new inner node with two new leaf nodes in its place. Alternatively, we may make the current inner node a leaf. To this end, we need to compare the loss at the inner node with the loss of the current subtree. The corresponding gain is
\begin{equation}
    G_{I_t} = \sum_{J_t \subseteq I_t} L(\Theta_{J_t}, Y_{J_t}, X_{J_t}) - L(\ThetasubI, \YsubI, \XsubI) \label{eq:gain_prune}.
\end{equation}
If both gains \eqref{eq:gain_inner} and \eqref{eq:gain_prune} are positive and $G_{I_t} \geq G_{I_t, C_t}$, we apply the second option, replacing the inner node with a leaf node, to obtain the overall smaller tree. Notably, \eqref{eq:gain_inner} and \eqref{eq:gain_prune} allow us to maintain the minimality of a Dynamic Model Tree:
\begin{lemma}\label{lemma:minimality}
    Greedy replacement of inner nodes, wherever $G_{I_t, C_t} \geq 0$ due to \eqref{eq:gain_inner} or $G_{I_t} \geq 0$ due to \eqref{eq:gain_prune}, implies \textit{model minimality} (Property \ref{prop:minimality}).
\end{lemma}
\begin{proof}
    Let $I_t$ represent an inner node of the Dynamic Model Tree. There exists a subtree whose root is the inner node $I_t$. We may prune this subtree by replacing the inner node $I_t$ with a different split candidate or a leaf. The gain $G$ corresponds to \eqref{eq:gain_inner} or \eqref{eq:gain_prune} respectively. Accordingly, $L_{alt}(\Omega_t) = L(\Omega_t) - G$ is the loss of the potential alternate tree with the subtree replaced. Note that the alternate tree is guaranteed to have an equal or lower number of nodes and, since the number of parameters per node is fixed, an equal or lower number of parameters. Since $L_{alt}(\Omega_t) = L(\Omega_t)$ implies that $G = 0$, by assumption we would replace the Dynamic Model Tree by the alternate tree with the smaller number of parameters. This procedure may be repeated from the bottom to the root of the tree.
\end{proof}

\subsection{Candidate Loss Approximation}\label{sec:approx_gains}
To compare the gains \eqref{eq:gain_leaf} or \eqref{eq:gain_inner} of different split candidates, we require loss estimates $L(\ThetasubC, \YsubC, \XsubC)$ and $L(\ThetasubbarC, \YsubbarC, \XsubbarC)$ for each candidate. However, due to limited resources, we usually cannot train the simple models corresponding to every potential split candidate. For this purpose, we adopt an efficient gradient-based approximation.

The authors in \cite{broelemann2018gradient} argue that we may warm-start optimizing the parameters of a split candidate $\ThetasubC$ with a single gradient step on the parameters of the current node $\ThetasubS$:
\begin{equation}\label{eq:initial_guess_param}
    \ThetasubC \approx \ThetasubS - \frac{\lambda}{|C_t|} \nabla_{\ThetasubS} L(\ThetasubS, \YsubC, \XsubC)
\end{equation}
The first order Taylor polynomial at the point $\ThetasubS$ then gives a good approximation to the loss of the split candidate $L(\ThetasubC, \YsubC, \XsubC)$. Accordingly, we write
\begin{align}\label{eq:approx_loss}
    L(\Theta, \YsubC, \XsubC) &\approx L(\ThetasubS, \YsubC, \XsubC) \nonumber \\
    &+ (\Theta - \ThetasubS)^T \nabla_{\ThetasubS} L(\ThetasubS, \YsubC, \XsubC) \nonumber \\
    \overset{\eqref{eq:initial_guess_param}}{\Rightarrow} L(\ThetasubC, \YsubC, \XsubC) &\approx L(\ThetasubS, \YsubC, \XsubC) \nonumber \\
    &- \frac{\lambda}{|C_t|} \| \nabla_{\ThetasubS} L(\ThetasubS, \YsubC, \XsubC) \|^2_2.
\end{align}
With \eqref{eq:approx_loss} we can approximate the loss of different split candidates without maintaining corresponding simple models. Moreover, we can reuse the gradient calculated during the optimisation of the parent model, which further increases efficiency. Finally, note that other work has successfully used gradient-based split finding \cite{gouk2019stochastic}.

\begin{algorithm}[t]
\caption{Dynamic Model Tree - General Update Procedure at a Leaf Node at Time Step $t$}
\label{alg:pseudo_code}
\begin{algorithmic}[1]
\renewcommand{\algorithmicrequire}{\textbf{Input:}}
\renewcommand{\algorithmicensure}{\textbf{Output:}}
\REQUIRE Observations and labels $X_t,Y_t$; Simple model $M_{\theta_t}$; Likelihoods, gradients and counts of time step $t-1$.
\ENSURE  Updated likelihoods, gradients and counts.
\\ \textit{***~Increment the loss, gradient and count at the node.~***}
\STATE $L(\ThetasubS,\YsubS,\XsubS) \leftarrow$ \\\hspace{2.3cm} $L(\Theta_{S_{t-1}},Y_{S_{t-1}},X_{S_{t-1}}) + L(\theta_t, Y_t, X_t)$
\STATE $\nabla_{\ThetasubS} L(\ThetasubS, \YsubS, \XsubS) \leftarrow$ \\\hspace{0.6cm} $\nabla_{\Theta_{S_{t-1}}} L(\Theta_{S_{t-1}}, Y_{S_{t-1}}, X_{S_{t-1}}) + \nabla_{\theta_t} L(\theta_t, Y_t, X_t)$
\STATE $n_{S_t} \leftarrow n_{S_{t-1}} + len(Y_t)$
\\ \textit{***~Update the statistics of split candidates and compute the gains (NOTE: The right child statistics corresponding to the set $\bar{C}_t$ can be obtained as the difference between the statistics of the left child ($C_t$) and the parent node ($S_t$). They therefore do not need to be stored separately.)-~***}
\STATE $G_{\max} \leftarrow -1$
\STATE $C_{\text{top}} \leftarrow \text{None}$
\FOR {all split candidates $C$}
    \STATE $Y^C_t \subseteq Y_t;~ X_t^C \subseteq X_t$
    \STATE $L(\ThetasubS,\YsubC,\XsubC) \leftarrow$ \\\hspace{1.6cm} $L(\Theta_{S_{t-1}},Y_{C_{t-1}},X_{C_{t-1}}) + L(\theta_t, Y_t^C, X_t^C)$
    \STATE $\nabla_{\ThetasubS} L(\ThetasubS, \YsubC, \XsubC) \leftarrow$ \\ $\nabla_{\Theta_{S_{t-1}}} L(\Theta_{S_{t-1}}, Y_{C_{t-1}}, X_{C_{t-1}}) + \nabla_{\theta_t} L(\theta_t, Y^C_t, X^C_t)$
    \STATE $n_{C_t} \leftarrow n_{C_{t-1}} + len(Y^C_t)$
    \item[]
    \STATE $L(\ThetasubC,\YsubC,\XsubC) \leftarrow \eqref{eq:approx_loss}$
    \STATE $G_{S_t, C_t} \leftarrow \eqref{eq:gain_leaf}$
    \IF {($G_{S_t, C_t} > G_{\max}$)}
        \STATE $G_{\max} \leftarrow G_{S_t, C_t}$
        \STATE $C_{\text{top}} \leftarrow C$
    \ENDIF
\ENDFOR
\\ \textit{***~Split or retain the leaf node.~***}
\IF {$G_{\max} \geq 0$}
    \STATE Split on candidate $C_{\text{top}}$
\ENDIF
\end{algorithmic}
\end{algorithm}
\subsection{Basic Algorithm And Complexity}\label{sec:complexity}
Algorithm \ref{alg:pseudo_code} depicts the general procedure at a leaf node of the Dynamic Model Tree. For inner nodes, we compute the gain functions \eqref{eq:gain_inner} and \eqref{eq:gain_prune} in line 12. In line 19, we then replace the inner node with a new split or a leaf (depending on which gain is greater). Otherwise, the general update procedure is equivalent for both types of nodes. We update the nodes of the tree in a bottom-up fashion.

The time complexity of Algorithm \ref{alg:pseudo_code} for updating one node without fitting the simple model is $\mathcal{O}(m n_t c + m^2 v c)$, where $c$ is the number of classes, $m$ is the number of features, $n_t$ is the sample size at time step $t$ and $v$ is the maximal number of unique values of a feature. Depending on the choice of simple model, the time complexity might increase. If the maximal number of unique values is large, i.e. $v \gg n_t$, then the first term becomes negligible, leading to a complexity of $\mathcal{O}(m^2 v c)$. In practice, decision tree algorithms often reduce computation time by limiting the number of eligible split candidates. This can be particularly important when we deal with large numbers of (continuous) features. We propose a simple method in Section \ref{sec:implementation}.

The memory complexity per node of the Dynamic Model Tree is $\mathcal{O}(m^2vc)$. As before, the memory requirements of the Dynamic Model Tree scale with the number of split candidates considered.

\subsection{Differences Between DMT and Earlier Methods}\label{sec:diff_fimt}
Dynamic Model Trees differ clearly from earlier work. A major difference lies in the way Dynamic Model Trees handle concept drift. While purity-based adaptation strategies usually require dedicated drift detection models to identify concept drift \cite{haug2021learning}, a Dynamic Model Tree does not. In fact, adaptation to concept drift is automatically handled via the proposed gain functions. As a consequence, Dynamic Model Trees only have few hyperparameters that need to be optimized, while providing a similar level of flexibility as earlier works.

FIMT-DD is one of the most popular existing Model Tree frameworks for data streams \cite{ikonomovska2011learning}. In the following, we briefly highlight key differences between FIMT-DD and the Dynamic Model Tree.
Like a Hoeffding Tree, FIMT-DD relies on a purity measure (Standard Deviation Reduction) and Hoeffding's inequality to compare split candidates. That is, in FIMT-DD, ``the process of learning linear models in the leaves will not explicitly reduce the size of the (...) tree'' \cite{ikonomovska2011learning}. In addition, FIMT-DD requires a dedicated concept drift detection method (Page Hinkley) to adapt to change. As mentioned before, the Dynamic Model Tree neither requires a heuristic measure nor a separate concept drift detection model.

Other than FIMT-DD, the Dynamic Model Tree continues to update the simple models at the inner nodes even after splitting. This may increase the computation time, but allows us to compute the loss concerning the active concept on different hierarchies. In this way, the proposed framework can effectively identify and adjust to global and local concept drift.

\subsection{Limitations}\label{sec:limitations}
Typically, incremental decision trees like VFDT \cite{domingos2000mining} or FIMT-DD \cite{ikonomovska2011learning} primarily occupy memory for saving statistics in the leaf nodes. Dynamic Model Trees also require memory to store statistics for every inner node. For example, while VFDT occupies $\mathcal{O}(lmvc)$ memory, a Dynamic Model Tree requires $\mathcal{O}((l+i)m^2vc)$, where $l$ and $i$ are the number of leaf and inner nodes, $m$ is the number of features, $v$ is the maximal number of unique values per feature and $c$ is the number of classes. However, Dynamic Model Trees usually remain shallow due to the \textit{model minimality} property, which reduces the overall computational gap to other methods.

Likewise, Dynamic Model Trees can have a longer training time per node, depending on the selected simple model type. The choice of appropriate simple models also affects the general performance of the tree. With random initial weights, a simple model may take some time to achieve good predictive quality. However, this mainly affects the root node of the Dynamic Model Tree, since all other simple models are warm-started with the optimized parameters of the parent node. In addition, if the simple models are non-robust or biased, the split and prune decisions of the proposed framework will suffer. In general, however, inadequate model types can be quickly identified by comparing the predictive error to benchmarks (e.g. the VFDT).

\section{Implementation}\label{sec:implementation}
The Dynamic Model Tree offers a large degree of flexibility. In particular, our framework may be implemented with different simple models and loss functions to account for different applications. For illustration, we propose an effective implementation of the Dynamic Model Tree for binary and multi-class classification.

\subsection{Simple Models}
We use logit and multinomial logit models (softmax) to represent binary and categorical target variables, respectively. Both models belong to the family of Generalized Linear Models (GLM) and are widely used in practice due to their efficiency and transparency. We train the simple models by stochastic gradient descent with a constant learning rate. In the future, one might experiment with different base models, optimization strategies or online feature selection \cite{haug2020leveraging}.

\subsection{Loss Function}
Owing to the proposed gain functions, changes in a Dynamic Model Tree are directly linked to changes in the empirical loss. Although purity-based splits usually also lead to a reduction in error, splits based on a change in loss can be very powerful in terms of interpretability.

To this end, we recall that concept drift between two time steps $t_1$ and $t_2$ corresponds to a change in the active concept, i.e. $P_{t_1}(X, Y) \neq P_{t_2}(X, Y)$. Online learning models need to adjust to concept drift in order to maintain high predictive performance. Accordingly, we are mainly interested in concept drift that shifts the optimal decision boundary. This form of concept drift is called real concept drift and is defined as $P_{t_1}(Y|X) \neq P_{t_2}(Y|X)$ \cite{lu2018learning}. Since the true distribution $P_t(Y|X)$ is generally unknown, our best approximation of the active data concept is the likelihood $P(Y_t|X_t,\theta_t)$ \cite{haug2021learning}. In this context, the negative log-likelihood $L(\ThetasubS, \YsubS, \XsubS) = -\sum_{t \in S_t} \log P(Y_t|X_t,\theta_t)$ is a straight-forward choice for the loss function.

If a simple model performs well, we can generally assume that the likelihood is a good approximation of the data-generating concept. Accordingly, we may assume that the negative log-likelihood loss changes as a consequence of concept drift. For this reason, the negative log-likelihood loss allows us to associate any (major) change in the gains \eqref{eq:gain_leaf}-\eqref{eq:gain_prune} with a local change in the approximate data concept. Compared to popular purity measures, this enables a much higher degree of online interpretability, as discussed in the introduction.

\subsection{Threshold for Robust Model Updates}
In practice, an online learning model will be subject to small variations and noise. It may therefore be useful to specify a threshold on the gain functions to avoid excessive updates. 

If we set a threshold for the gains defined in \eqref{eq:gain_inner} and \eqref{eq:gain_prune}, we need to relax the \textit{model minimality} (Property \ref{prop:minimality}): We recall that the loss of a minimal alternate model is given by $L_{alt}(\Omega_t) = L(\Omega_t) - G$, where $L(\Omega_t)$ is the loss of the current tree and $G$ corresponds to \eqref{eq:gain_inner} or \eqref{eq:gain_prune} (see Lemma \ref{lemma:minimality} and Proof). Consequently, if we prune the inner node whenever $G \geq threshold \geq 0$, we retain the minimal model for $L_{alt}(\Omega_t) \leq L(\Omega_t) - threshold$. This relaxation can sometimes be sensible, since a non-robust tree may be equally undesirable than an overly complex tree. Besides, if the threshold is reasonably small, changes of the loss due to concept drift will usually trigger model updates after a few iterations. To set a threshold, one only needs to adjust line 18 of the basic procedure shown in Algorithm \ref{alg:pseudo_code}.

By using the negative log-likelihood loss, we enable a natural threshold in terms of the Akaike Information Criterion:
\begin{equation}
    AIC = 2k - 2 \ell(\Theta),
\end{equation}
where $\ell$ is the log-likelihood and $k$ is the number of free (estimated) parameters. The AIC is a popular test statistic for model selection problems. It estimates the relative amount of information lost among competing models. Given two models $i$ and $j$ where $AIC_i \leq AIC_j$, the quantity $\exp([AIC_i - AIC_j]/2)$ is proportional to the relative probability that model $j$ minimizes the estimated information loss. Therefore, if we set a threshold for this quantity, we can control the tolerated probability that model $j$ actually has the minimum AIC instead of model $i$. 

We can apply this methodology to our split and prune strategy. For example, when attempting to split, we compare the simple models representing the current node ($S_t$) and the potential split ($C_t$, $\bar{C}_t$). The corresponding AICs are
\begin{align}
    AIC_{S_t} &= 2k_{S_t}  - 2 \ell(\ThetasubS, \YsubS, \XsubS), \\
    AIC_{C_t} &= 2 (k_{C_t} + k_{\bar{C}_t}) \nonumber \\
    &- 2 \big(\ell(\ThetasubC, \YsubC, \XsubC) + \ell(\ThetasubbarC, \YsubbarC, \XsubbarC)\big),
\end{align}
where $k_{S_t}$, $k_{C_t}$ and $k_{\bar{C}_t}$ denote the numbers of free parameters of the corresponding models. Let $\epsilon \in [0,1]$ be a user-specified hyperparameter. We apply the following test:
\begin{align}\label{eq:threshold_leaf}
    &\exp([AIC_{C_t} - AIC_{S_t}] / 2) \leq \epsilon \nonumber \\
    \Leftrightarrow~ &\exp\big(k_{C_t} + k_{\bar{C}_t} - \ell(\ThetasubC, \YsubC, \XsubC) - \ell(\ThetasubbarC, \YsubbarC, \XsubbarC) \nonumber \\
    &~~~- k_{S_t}  + \ell(\ThetasubS, \YsubS, \XsubS) \big) \leq \epsilon \nonumber \\
    \overset{\eqref{eq:gain_leaf}}{\Leftrightarrow}~ &\exp\big(k_{C_t} + k_{\bar{C}_t} - k_{S_t} - G_{S_t,C_t} \big) \leq \epsilon \nonumber \\
    \Leftrightarrow~ &\exp(-G_{S_t,C_t}) \leq \frac{\epsilon}{\exp(k_{C_t} + k_{\bar{C}_t} - k_{S_t})} \nonumber \\
    \Leftrightarrow~ &-G_{S_t,C_t} \leq \log(\epsilon) - k_{C_t} - k_{\bar{C}_t} + k_{S_t} \nonumber \\
    \Leftrightarrow~ &G_{S_t,C_t} \geq k_{C_t} + k_{\bar{C}_t} - k_{S_t} - \log(\epsilon)
\end{align}
If we use the same simple model type at every node (e.g. logit models as proposed earlier), then \eqref{eq:threshold_leaf} simplifies to $G_{S_t,C_t} \geq k - \log(\epsilon)$. Similarly, we can calculate thresholds for the remaining gain functions, which we omit for brevity. Note that the hyperparameter $\epsilon$ controls the trade-off between quick and robust updates. In this way, we can adjust the sensitivity of the Dynamic Model Tree.

\subsection{Algorithmic Considerations}
We implemented the Dynamic Model Tree in Python.\footnote{\url{https://github.com/haugjo/dynamic-model-tree}} Note that the Dynamic Model Tree is able to handle both batch-incremental and instance-incremental online learning. In the following, we discuss important algorithmic details and propose a sensible hyperparameter configuration.

In practice, the number of unique split candidates may grow quickly -- in particular for continuous variables. This is a problem that most incremental decision trees have in common. To overcome potential memory overload, our framework may be extended with advanced strategies like Binary Search Trees (see their application in FIMT-DD \cite{ikonomovska2011learning}, for example). For illustration, however, we have chosen a simpler technique. 

Specifically, we store only a fixed number of statistics corresponding to the candidates with largest estimated gain (we recommend a default value of three times the number of features). At every time step, we allow a fixed percentage of the saved candidate statistics to be replaced by newly observed candidates. This is similar to the VFDT algorithm \cite{domingos2000mining}, which drops split candidates that diverge too far from the current maximal gain. We recommend a default replacement rate of 50\%, which provided good results throughout all our experiments.

Since we limit the number of split candidates in main memory, we need to approximate the gain of newly observed candidates from the current sample. Note that the initial approximation can be biased if the current batch is not representative of the active concept. Specifically, such initial bias might occur if the batch size is small or the data is very noisy. Once stored, however, the statistics are updated at each successive time step, mitigating any initial bias over time. In addition, a split candidate that was rejected or deleted in the past can be added again in the future, e.g. if its importance has changed after concept drift. In experiments, we obtained good results for this simple approximation scheme.

Additionally, we propose a learning rate of 0.05 to train the binary and multinomial logit models and a threshold of $\epsilon = 10e-8$ for the AIC-based confidence test.

Finally, note that we might be able to improve the efficiency of the Dynamic Model Tree by using parallelization or distributed computation. We leave a detailed discussion of more advanced implementation techniques for future work.

\begin{table}[t]
\caption{\emph{Data sets}. We used state-of-the-art tabular streaming data sets with different types of concept drift. TüEyeQ \cite{kasneci2021tueyeq}, as well as Insects-Abrupt and Insects-Incremental \cite{souza2020challenges} have been obtained from the sources referenced in the respective papers. The remaining real-world data sets have been obtained from \textit{https://www.openml.org}. We included the original reference wherever available. The synthetic data sets have been generated with \textit{scikit-multiflow} \cite{montiel2018scikit}. Here we also indicate the type of concept drift generated (abrupt or incremental).}
    \label{tab:datasets}
    \centering
    \begin{adjustbox}{max width=\columnwidth}
        \begin{tabular}{llll}
        \toprule
        \textbf{Name} & \textbf{\#Samples} & \textbf{\#Features} & \textbf{\#Classes (\#Majority)} \\ 
        \cmidrule(lr){1-1} \cmidrule(lr){2-4}
        Electricity & 45,312 & 8 & 2 -- (26,075) \\
        Airlines & 539,383 & 7 & 2 -- (299,119) \\
        Bank \cite{moro2011using} & 45,211 & 16 & 2 -- (39,922) \\
        TüEyeQ \cite{kasneci2021tueyeq} & 15,762 & 76 & 2 -- (12,975) \\
        Poker-Hand \cite{Dua2019} & 1,025,000 & 10 & 9 -- (513,701) \\
        KDDCup & 494,020 & 41 & 23 -- (280,790) \\
        Covertype \cite{Dua2019} & 581,012 & 54 & 7 -- (283,301) \\
        Gas \cite{vergara2012chemical} & 13,910 & 128 & 6 -- (3,009) \\
        Insects-Abrupt \cite{souza2020challenges} & 355,275 & 33 & 6 -- (101,256) \\
        Insects-Incremental \cite{souza2020challenges} & 452,044 & 33 & 6 -- (134,717) \\
        SEA (synthetic, abrupt) & 1,000,000 & 3 & 2 \\
        Agrawal (synthetic, incremental) & 1,000,000 & 9 & 2 \\
        Hyperplane (synthetic, incremental) & 500,000 & 50 & 2 \\
        \bottomrule
        \end{tabular}
    \end{adjustbox}
\end{table}

\begin{table*}[!ht]
\caption{\textit{F1 Measure (higher is better).} We show the mean and standard deviation of the F1 measures observed over time in all data sets. For reference, we also provide the results of two ensemble classifiers separated from the stand-alone models by horizontal lines. We highlight the top result of each data set in bold letters. The average performance across all data sets is shown in the rightmost column. Note that the standard deviation also captures the variation caused by concept drift. It should therefore not be taken as an indication of the robustness to noise. The proposed Dynamic Model Tree frequently outperforms the remaining classifiers in terms of the predictive power and performs best on average.}
    \label{tab:f1}
    \centering
    \begin{adjustbox}{max width=\textwidth}
        \begin{tabular}{lllllllllllllll}
        \toprule
        Model $\backslash$ Data Set &  Electricity &     Airlines & Bank &        TüEyeQ &        Poker &          KDD &    Covertype &          Gas & Insects-Abr. &  Insects-Inc. &          SEA &      Agrawal &   Hyperplane &         Mean \\
        \cmidrule(lr){1-1} \cmidrule(lr){2-14} \cmidrule(lr){15-15}
        DMT (ours)    &  0.76 ± 0.20 &  0.63 ± 0.05 &    \textbf{0.88 ± 0.11} &  \textbf{0.79 ± 0.20} &  0.44 ± 0.05 &  \textbf{0.99 ± 0.01} &  0.80 ± 0.09 &  \textbf{0.82 ± 0.27} &    \textbf{0.73 ± 0.10} &  \textbf{0.73 ± 0.08} &  0.88 ± 0.02 &  0.82 ± 0.08 &  \textbf{0.84 ± 0.04} &  \textbf{0.78 ± 0.10} \\
        FIMT-DD \cite{ikonomovska2011learning}  &  0.78 ± 0.20 &  0.55 ± 0.12 &    \textbf{0.88 ± 0.14} &  0.76 ± 0.22 &  0.41 ± 0.08 &  \textbf{0.99 ± 0.01} &  0.81 ± 0.10 &  0.79 ± 0.28 &   \textbf{0.73 ± 0.08} &  0.72 ± 0.08 &  0.78 ± 0.10 &  0.64 ± 0.13 &  0.76 ± 0.05 &  0.74 ± 0.12 \\
        VFDT (MC) \cite{domingos2000mining}  &  0.76 ± 0.20 &  0.64 ± 0.06 &    0.87 ± 0.15 &  0.77 ± 0.22 &  0.47 ± 0.05 &  0.96 ± 0.10 &  0.72 ± 0.13 &  0.29 ± 0.37 &    0.64 ± 0.14 &  0.67 ± 0.10 &  0.86 ± 0.03 &  0.77 ± 0.11 &  0.65 ± 0.03 &  0.70 ± 0.13 \\
        VFDT (NBA) \cite{gama2003accurate}  &  \textbf{0.80 ± 0.15} &  \textbf{0.65 ± 0.05} &    \textbf{0.88 ± 0.13} &  0.77 ± 0.21 &  \textbf{0.50 ± 0.03} &  \textbf{0.99 ± 0.01} &  \textbf{0.85 ± 0.09} &  0.77 ± 0.27 &    0.71 ± 0.10 &  0.72 ± 0.07 &  0.86 ± 0.04 &  0.79 ± 0.10 &  0.73 ± 0.02 &  0.77 ± 0.10 \\
        HT-ADA \cite{bifet2009adaptive}  &  0.77 ± 0.21 &  0.62 ± 0.07 &    \textbf{0.88 ± 0.13} &  0.77 ± 0.23 &  0.47 ± 0.05 &  0.96 ± 0.10 &  0.67 ± 0.19 &  0.22 ± 0.35 &    0.59 ± 0.15 &  0.64 ± 0.13 &  \textbf{0.89 ± 0.02} &  \textbf{0.84 ± 0.08} &  0.66 ± 0.03 &  0.69 ± 0.13 \\
        EFDT \cite{manapragada2018extremely}   &  0.77 ± 0.20 &  0.60 ± 0.09 &    \textbf{0.88 ± 0.14} &  0.77 ± 0.23 &  0.47 ± 0.05 &  \textbf{0.99 ± 0.01} &  0.74 ± 0.14 &  0.55 ± 0.39 &    0.68 ± 0.11 &  0.65 ± 0.10 &  0.87 ± 0.04 &  0.82 ± 0.09 &  0.69 ± 0.03 &  0.73 ± 0.12 \\
        \midrule\midrule
        Forest Ens. \cite{gomes2017adaptive}  &  \textbf{0.81 ± 0.14} &  0.64 ± 0.05 &    \textbf{0.89 ± 0.13} &  \textbf{0.78 ± 0.20} &  0.50 ± 0.02 &  \textbf{0.99 ± 0.01} &  \textbf{0.74 ± 0.19} &  \textbf{0.80 ± 0.33} &    0.72 ± 0.09 &  0.72 ± 0.08 &  \textbf{0.90 ± 0.02} &  0.80 ± 0.08 &  0.64 ± 0.03 &  0.76 ± 0.10 \\
        Bagging Ens. \cite{bifet2010leveraging} &  \textbf{0.81 ± 0.17} &  \textbf{0.65 ± 0.05} &    \textbf{0.89 ± 0.13} &  \textbf{0.78 ± 0.21} & \textbf{0.53 ± 0.03} &  \textbf{0.99 ± 0.04} &  0.72 ± 0.23 &  0.67 ± 0.40 &    \textbf{0.74 ± 0.10} &  \textbf{0.75 ± 0.07} &  \textbf{0.90 ± 0.02} &  \textbf{0.84 ± 0.08} &  \textbf{0.72 ± 0.04} &  \textbf{0.77 ± 0.12} \\
        \bottomrule
        \end{tabular}
    \end{adjustbox}
\end{table*}

\begin{table*}[!ht]
\caption{\textit{No. of Splits (lower is better).} Complexity -- quantified here by the mean and standard deviation of the number of splits (as described in Section \ref{sec:interpretability_measure}) -- is often used as an indicator of the interpretability of a model. Model Trees (FIMT-DD and DMT) tend to remain shallower than Hoeffding Trees, due to the flexibility provided by the linear leaf models.}
    \label{tab:n_splits}
    \centering
    \begin{adjustbox}{max width=\textwidth}
        \begin{tabular}{lllllllllllllll}
        \toprule
        Model $\backslash$ Data Set &  Electricity &     Airlines & Bank &        TüEyeQ &        Poker &          KDD &    Covertype &          Gas & Insects-Abr. &  Insects-Inc. &          SEA &      Agrawal &   Hyperplane &         Mean \\
        \cmidrule(lr){1-1} \cmidrule(lr){2-14} \cmidrule(lr){15-15}
        DMT (ours)     &    6.5 ± 3.1 &    35.7 ± 16.7 &      \textbf{2.3 ± 1.0} &    1.4 ± 0.8 &     \textbf{ 9.0 ± 0.0} &     24.8 ± 6.3 &       10.7 ± 4.0 &   9.3 ± 3.5 &      9.1 ± 3.5 &      \textbf{9.1 ± 3.5} &     35.1 ± 25.3 &     75.4 ± 34.4 &      \textbf{2.2 ± 1.3} &     \textbf{17.7 ± 8.0} \\
        FIMT-DD \cite{ikonomovska2011learning} &  52.0 ± 30.1 &      \textbf{4.9 ± 3.9} &    75.5 ± 47.3 &    \textbf{1.0 ± 0.0} &    17.7 ± 10.2 &     24.8 ± 6.4 &       13.7 ± 8.2 &   6.0 ± 0.0 &      \textbf{7.4 ± 3.1} &     10.6 ± 5.9 &       \textbf{1.0 ± 0.0} &     \textbf{65.8 ± 71.5} &     8.0 ± 10.3 &    22.2 ± 15.1 \\
        VFDT (MC) \cite{domingos2000mining}   &  37.8 ± 22.3 &  323.3 ± 182.4 &    21.9 ± 13.9 &   10.6 ± 6.8 &    84.7 ± 50.6 &    25.6 ± 13.0 &    356.8 ± 201.7 &   0.7 ± 0.7 &    41.3 ± 23.7 &    53.5 ± 32.5 &   588.4 ± 339.8 &   628.3 ± 371.0 &  277.9 ± 162.4 &  188.5 ± 109.3 \\
        VFDT (NBA) \cite{gama2003accurate}  &  76.7 ± 44.6 &  647.6 ± 364.7 &    44.8 ± 27.7 &  22.3 ± 13.7 &  856.3 ± 506.0 &  637.3 ± 310.8 &  2861.1 ± 1613.4 &  11.1 ± 5.1 &  295.2 ± 165.7 &  380.3 ± 227.6 &  1177.8 ± 679.7 &  1257.6 ± 742.1 &  556.8 ± 324.9 &  678.8 ± 386.6 \\
        HT-ADA \cite{bifet2009adaptive}  &    \textbf{3.4 ± 2.1} &     12.7 ± 6.8 &      5.6 ± 3.4 &    2.3 ± 1.6 &    58.0 ± 28.1 &    25.4 ± 12.8 &        \textbf{3.1 ± 2.9} &   \textbf{0.2 ± 0.4} &      8.0 ± 5.0 &    21.5 ± 12.9 &    131.4 ± 69.8 &    158.2 ± 79.2 &  188.7 ± 101.4 &    47.6 ± 25.1 \\
        EFDT \cite{manapragada2018extremely}   &   10.9 ± 4.5 &     15.2 ± 7.5 &      9.5 ± 3.4 &    2.8 ± 1.4 &     10.0 ± 6.6 &     \textbf{24.7 ± 9.2} &        9.4 ± 4.3 &   4.7 ± 2.7 &     17.3 ± 7.8 &    15.9 ± 10.4 &    109.9 ± 70.3 &     89.7 ± 66.2 &    31.0 ± 17.4 &    27.0 ± 16.3 \\
        \bottomrule
        \end{tabular}
    \end{adjustbox}
\end{table*}

\begin{table*}[!ht]
\caption{\textit{No. of Parameters (lower is better).} For the sake of completeness and to account for the difference between majority weighting and linear leaf models, we depict the number of parameters (mean ± standard deviation) as another measure of complexity (as described in Section \ref{sec:interpretability_measure}). In general, heuristic measures like the number of splits or parameters do not always give a clear indication of the interpretability of a model and should thus be considered with care. A more reliable indication of interpretability is provided by theoretical properties such as Property \ref{prop:conistency} and \ref{prop:minimality}.}
    \label{tab:n_param}
    \centering
    \begin{adjustbox}{max width=\textwidth}
        \begin{tabular}{lllllllllllllll}
        \toprule
        Model $\backslash$ Data Set &  Electricity &     Airlines & Bank &        TüEyeQ &        Poker &          KDD &    Covertype &          Gas & Insects-Abr. &  Insects-Inc. &          SEA &      Agrawal &   Hyperplane &         Mean \\
        \cmidrule(lr){1-1} \cmidrule(lr){2-14} \cmidrule(lr){15-15}
        DMT (ours)    &     33 ± 14 &     146 ± 67 &         27 ± 8 &    92 ± 31 &       80 ± 0 &      970 ± 238 &       474 ± 162 &   939 ± 320 &       237 ± 82 &     238 ± 82 &      71 ± 51 &    381 ± 172 &        80 ± 33 &      290 ± 97 \\
        FIMT-DD \cite{ikonomovska2011learning}  &   238 ± 136 &      \textbf{22 ± 15} &      649 ± 402 &     76 ± 0 &     150 ± 83 &      971 ± 239 &       597 ± 332 &     640 ± 0 &       198 ± 74 &    275 ± 140 &        \textbf{3 ± 0} &    333 ± 358 &      229 ± 262 &     337 ± 157 \\
        VFDT (MC) \cite{domingos2000mining}  &     77 ± 45 &    648 ± 365 &        45 ± 28 &    22 ± 14 &    170 ± 101 &        52 ± 26 &       715 ± 403 &       2 ± 1 &        84 ± 47 &     108 ± 65 &   1178 ± 680 &   1258 ± 742 &      557 ± 325 &     378 ± 219 \\
        VFDT (NBA) \cite{gama2003accurate} &   349 ± 201 &  2,594 ± 1,459 &      388 ± 236 &  896 ± 526 &  6,943 ± 4,099 &  24,016 ± 11,695 &  116,270 ± 65,543 &  1,105 ± 470 &    7,023 ± 3,930 &  9,042 ± 5,397 &  2,357 ± 1,359 &  6,292 ± 3,710 &   14,224 ± 8,285 &  14,731 ± 8,224 \\
        HT-ADA \cite{bifet2009adaptive}  &       \textbf{8 ± 4} &      27 ± 14 &         \textbf{12 ± 7} &      \textbf{6 ± 3} &     144 ± 78 &        52 ± 26 &           \textbf{7 ± 6} &       \textbf{1 ± 1} &        \textbf{17 ± 10} &      44 ± 26 &    264 ± 140 &    377 ± 193 &      378 ± 203 &      103 ± 55 \\
        EFDT \cite{manapragada2018extremely}   &      23 ± 9 &      31 ± 15 &         20 ± 7 &      7 ± 3 &      \textbf{21 ± 13} &       \textbf{50 ± 18} &          20 ± 9 &      10 ± 5 &        36 ± 16 &      \textbf{33 ± 21} &    221 ± 141 &    \textbf{180 ± 132} &        \textbf{63 ± 35} &       \textbf{55 ± 33} \\
        \bottomrule
        \end{tabular}
    \end{adjustbox}
\end{table*}

\section{Experiments}\label{sec:experiments}
We evaluated the Dynamic Model Tree in multiple experiments on synthetic and real-world streaming classification data sets. Specifically, we compared the proposed framework to the related Model Tree architecture FIMT-DD \cite{ikonomovska2011learning} and different versions of the Hoeffding Tree. We begin with a description of the experimental setup, including the data sets, related methods and performance measures. Afterwards, we summarize our most important findings.

\subsection{Environment and Evaluation Strategy}
All models and experiments were implemented in Python (3.8.5) and run on an AMD Ryzen Threadripper 3960X (24x 3.8GHz) CPU with 128Gb RAM under Ubuntu 18.04. In addition, we used the following packages: numpy (1.20.1), pandas (1.2.4), matplotlib (3.4.2), scikit-learn (0.24.2) and scikit-multiflow (0.5.3). We specified a random state to guarantee the reproducibility of all results.

We performed a prequential (test-then-train) evaluation \cite{Gama2009}, which is the most common evaluation strategy for data stream learning. A disadvantage of data stream evaluations compared to regular batch evaluations is the lack of statistical significance. To be precise, since we cannot alter the order of observations without introducing artificial concept drift, we cannot obtain results for different permutations or samples of the data set. There are approaches where multiple instances of a classifier are trained in parallel \cite{Bifet2015}. However, they are very computationally intensive. Accordingly, we ask readers to be aware that statistical significance, although being standard in other areas of machine learning, is uncommon in the data stream literature.

At each iteration of the prequential evaluation, we processed a batch of $0.1\%$ of the data. We also examined other batch sizes to ensure that the reported results are representative.

\subsection{Data Sets}
Typically, online classifiers are evaluated on tabular data sets. Machine learning with heterogeneous and evolving tabular data is challenging and has recently attracted attention in other areas such as deep learning \cite{borisov2021deep}. In our experiments, we used state-of-the-art tabular streaming data sets, which we briefly describe in the following. We obtained most real-world data sets from \textit{https://www.openml.org}. A summary of the data sets and their properties can also be found in Table \ref{tab:datasets}.

The Electricity data set describes price changes in the Australian New South Wales Electricity Market. The prices are not fixed, but adjust over time to the varying supply and demand. In the Airlines data set, the goal is to predict whether a flight will be delayed, given information about its scheduled departure. The Bank Marketing data set incorporates information about a marketing campaign of a Portuguese bank institute \cite{moro2011using}. Here, the goal is to predict whether a customer will subscribe a deposit. Poker-Hand is a popular multiclass classification data set that consists of variables describing different poker hands \cite{Dua2019}. Covertype contains information about several forest cover types that need to be distinguished \cite{Dua2019}. The Gas data set contains drifting measurements of chemical sensors that are used to classify different types of gas \cite{vergara2012chemical}. The KDD Cup 1999 data set was introduced as part of a data mining competition. The data set contains features about network connections that are used for intrusion detection. We shuffled the KDD data set, because it was initially grouped by class labels. Since KDD does not involve known concept drift, shuffling the data is required to obtain an even distribution of classes over time and enable a fair evaluation.

It is usually difficult to determine the exact period of concept drift in a real-world streaming process. In fact, we cannot access such information for any of the above-mentioned data sets. Two recent exceptions are the TüEyeQ \cite{kasneci2021tueyeq} and Insects \cite{souza2020challenges} data collections. From TüEyeQ, we used the sociodemographic data about all subjects participating in an IQ test. The classification task is to decide whether a subject fails or passes an IQ-related task. The data set is divided in four task blocks with increasing difficulty within each block, resembling a natural concept drift. The Insects data comprises sensor information from monitoring of flying insect species. The measurements were obtained in a non-stationary but controllable environment. That is, by changing the temperature and humidity, the authors in \cite{souza2020challenges} were able to generate different types of concept drift. We used the imbalanced Insects data sets with abrupt and incremental drift.

In addition, we created synthetic data streams with \textit{scikit-multiflow} \cite{montiel2018scikit}. Specifically, we used the AGRAWALGenerator, HyperplaneGenerator and SEAGenerator to obtain synthetic data with different types of concept drift. For detailed information about each data generator, we refer to the corresponding documentation. Each synthetic data stream was sampled with 0.1 probability of noisy inputs (this corresponds to the ``perturbation'' parameter of the scikit-multiflow classes).

The resulting Hyperplane data set is subject to a continuous incremental concept drift over all observations. The Agrawal data set contains incremental drift between the observations 100,000-200,000, 300,000-500,000 and 800,000-900,000, but is otherwise stable. The SEA data set has three abrupt concept drifts at the observations 200,000, 400,000, 600,000 and 800,000.

Finally, we factorised the categorical string variables of all data sets. In addition, we normalized the features before use (range $[0,1]$). Otherwise, we did not pre-process the data sets.

\subsection{Related Algorithms and Hyperparameters}
As mentioned before, we compared the Dynamic Model Tree to different versions of the Hoeffding Tree. Specifically, we obtained results for the basic VFDT \cite{domingos2000mining} and two of its extensions, the adaptive Hoeffding Tree (HT-Ada) \cite{bifet2009adaptive} and the Extremely Fast Decision Tree (EFDT) \cite{manapragada2018extremely}. Unlike VFDT, both extensions contain a mechanism to adapt to concept drift.

Since it is generally not possible to optimize hyperparameters in a data stream, we applied the default configurations suggested by the corresponding scikit-multiflow implementations. These implementations have been heavily optimized over the years. Since our goal was to compare the originally proposed models, and in order to allow a fairer comparison with our implementation, we disabled some of the optimizations of scikit-multiflow. In particular, we did not use bootstrap sampling in the leaves of the HT-Ada algorithm. Moreover, we used majority voting in the leaf nodes of the Hoeffding Trees. However, to give an indication of the possible improvement introduced by simple predictive models in the leaves of a Hoeffding Tree, we also report the results of a VFDT augmented with adaptive Na\"ive Bayes models \cite{gama2003accurate}. Finally, to improve the efficiency of the EFDT algorithm, we set the minimum number of observations between re-evaluations to 1,000.

For the sake of completeness, we also looked at two state-of-the-art ensembles of the Hoeffding Tree, an Adaptive Random Forest \cite{gomes2017adaptive} and a Leveraging Bagging Ensemble \cite{bifet2010leveraging}. Both ensembles were trained with 3 basic Hoeffding Tree classifiers as weak learners. We configured the weak learners in the same way as the stand-alone VFDT model. Otherwise, we used the default parameters of the ensembles specified in scikit-multiflow.

In addition, we evaluated FIMT-DD \cite{ikonomovska2011learning}. To the best of our knowledge, there is no publicly available Python implementation of a FIMT-DD classification model. Therefore, we implemented the classifier based on the description in the paper.\footnote{The FIMT-DD implementation can also be accessed via Github at \url{https://github.com/haugjo/dynamic-model-tree}} Our implementation uses the second drift adjustment strategy proposed by the authors, i.e., it deletes branches where the Page Hinkley test issues an alert. We used a default learning rate of 0.01 for the simple models and a threshold of 0.01 for the significance test based on Hoeffding's inequality. Besides, we specified a threshold of 0.05 to break ties between split candidates with similar gain.

Finally, note that we only allowed binary splits in all incremental decision trees. The Dynamic Model Tree was configured in the way described in Section \ref{sec:implementation}.

\begin{figure*}[h!]
\centering
\subfloat[Hyperplane (Incremental Drift), F1 Measure]{
        \includegraphics[width=0.48\textwidth]{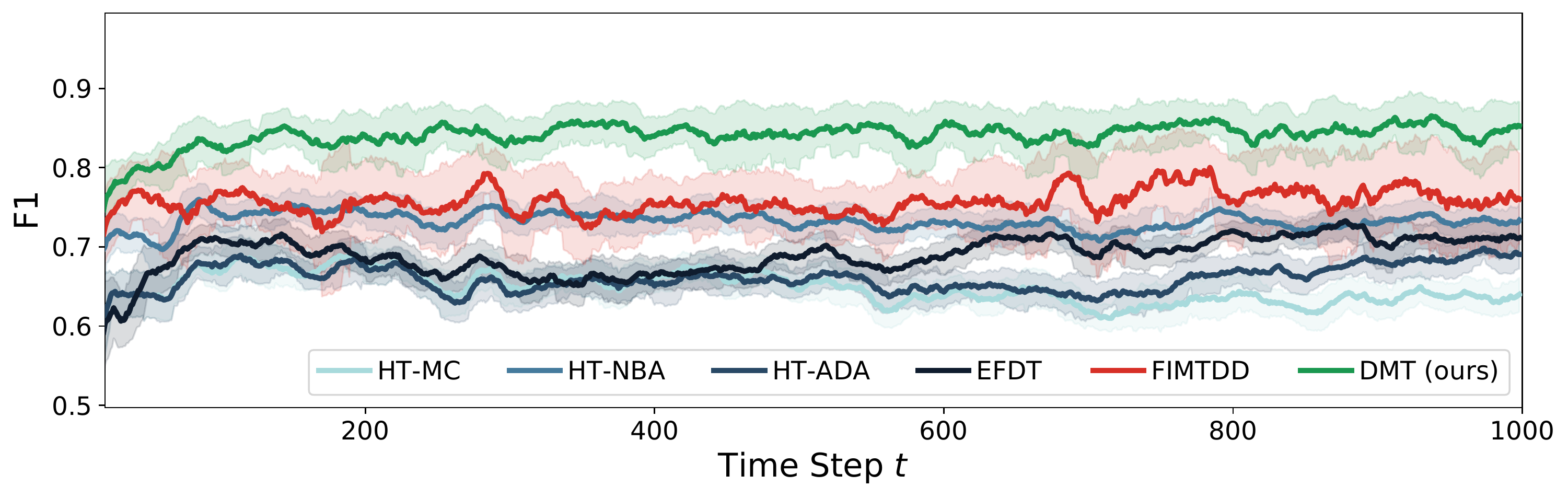}
}
\hfill
\subfloat[Hyperplane (Incremental Drift), Log Number of Splits]{
        \includegraphics[width=0.48\textwidth]{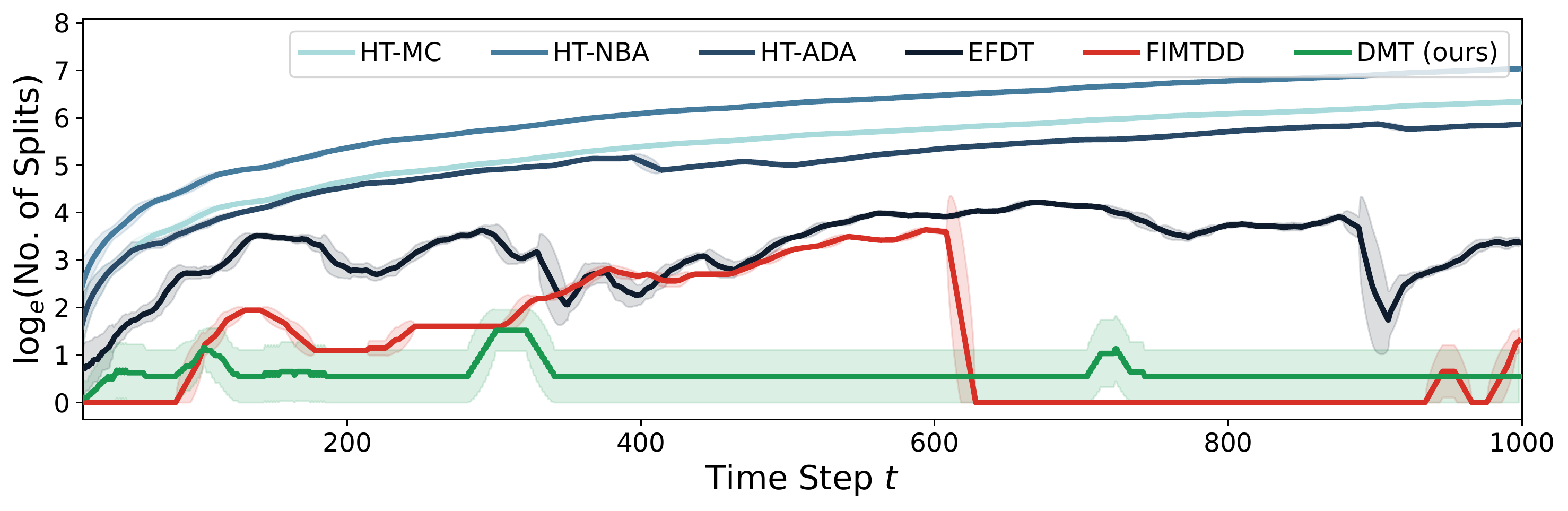}
}
\hfill
\subfloat[SEA (Abrupt Drifts), F1 Measure]{
        \includegraphics[width=0.48\textwidth]{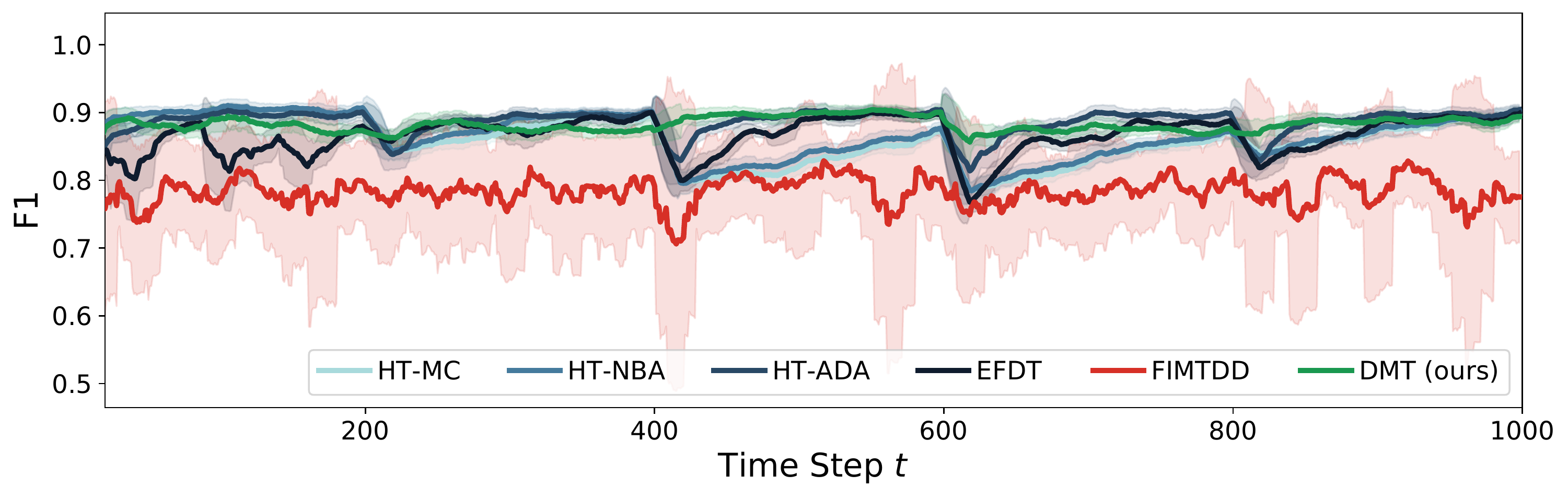}
}
\hfill
\subfloat[SEA (Abrupt Drifts), Log Number of Splits]{
        \includegraphics[width=0.48\textwidth]{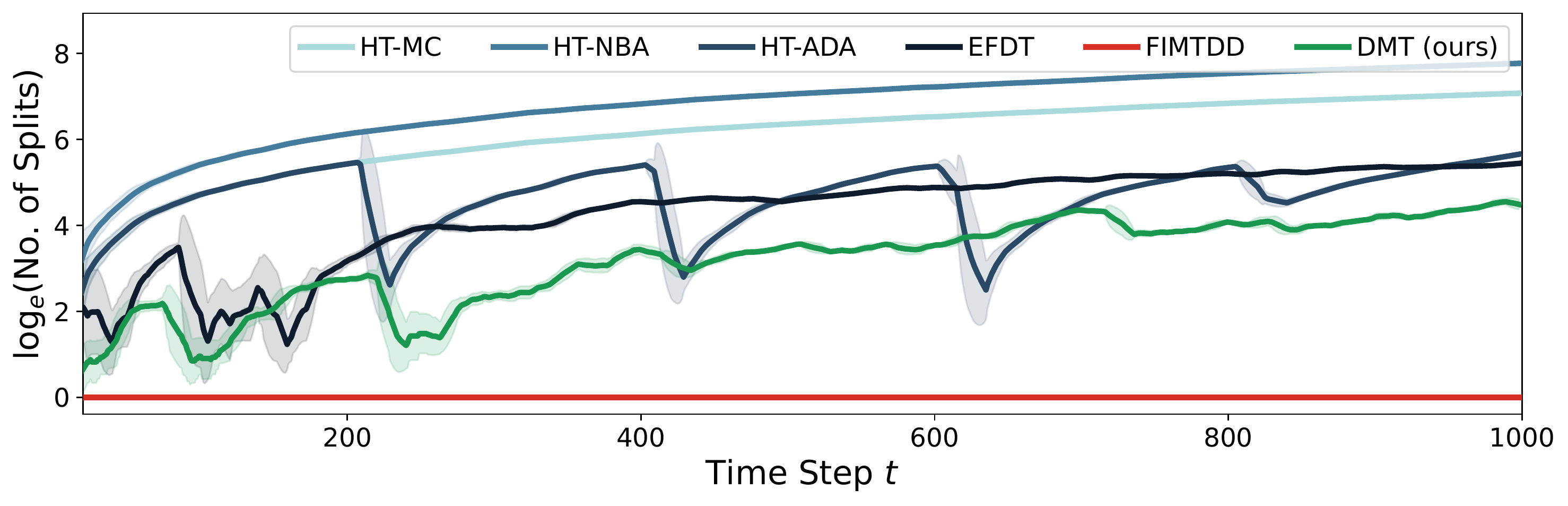}
}
\hfill
\subfloat[Insects-Inc. (Incremental Drift), F1 Measure]{
        \includegraphics[width=0.48\textwidth]{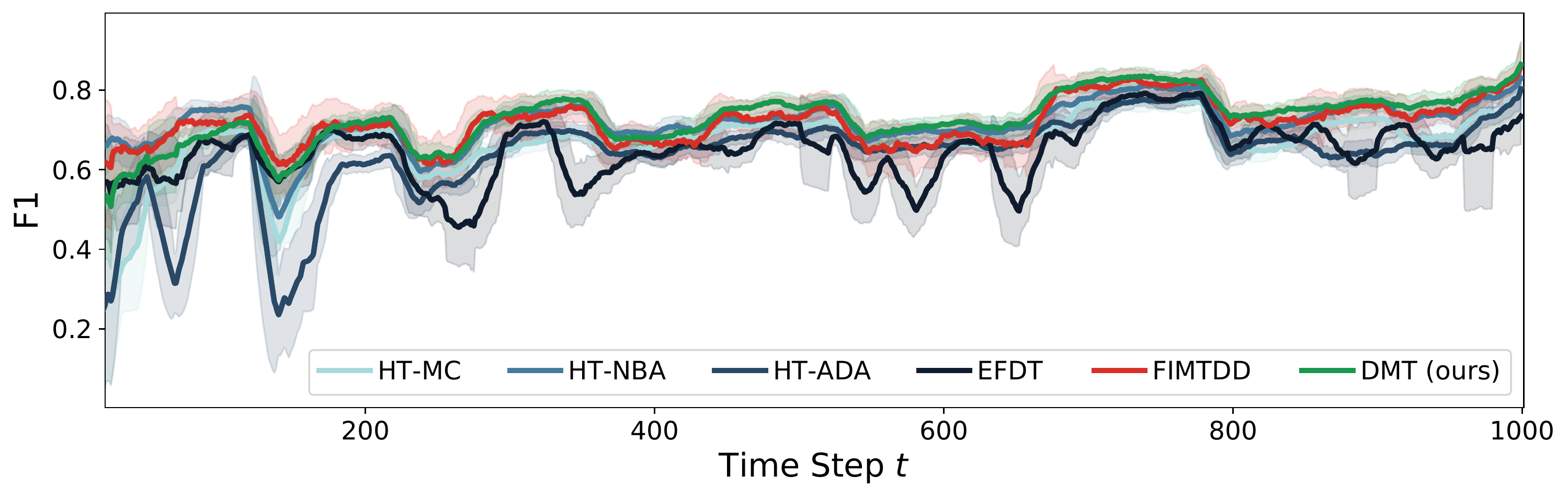}
}
\hfill
\subfloat[Insects-Inc. (Incremental Drift), Log Number of Splits]{
        \includegraphics[width=0.48\textwidth]{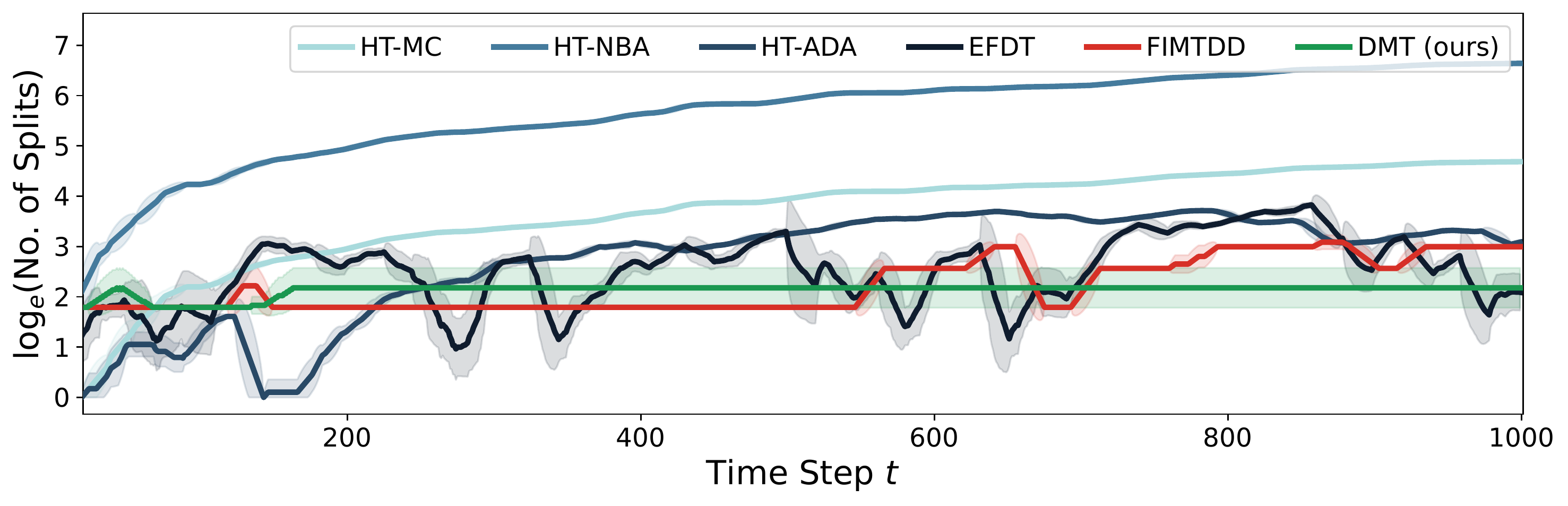}
}
\hfill
\subfloat[TüEyeQ (Abrupt Drifts), F1 Measure]{
        \includegraphics[width=0.48\textwidth]{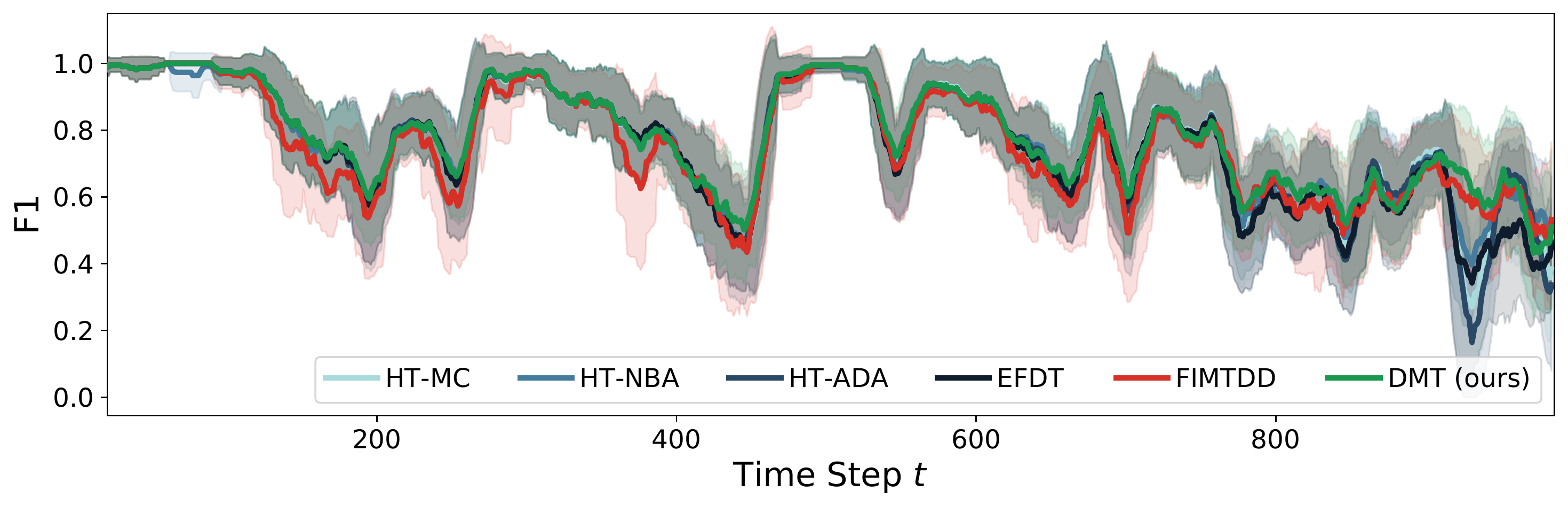}
}
\hfill
\subfloat[TüEyeQ (Abrupt Drifts), Log Number of Splits]{
        \includegraphics[width=0.48\textwidth]{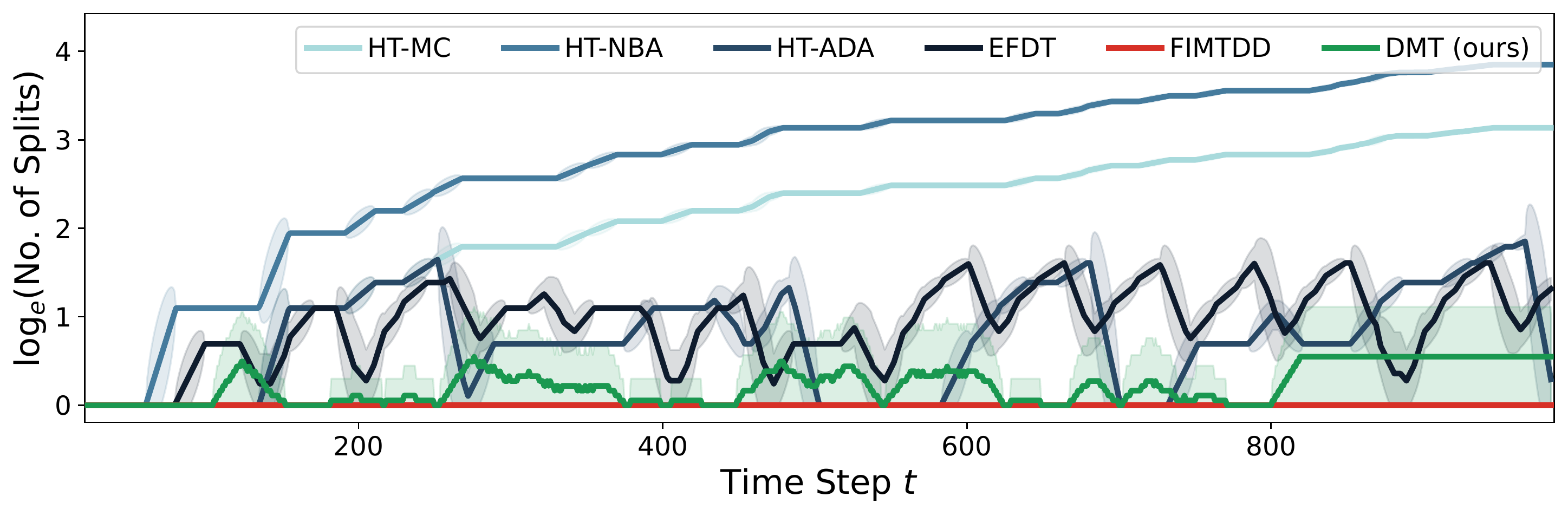}
}
\captionsetup{width=\linewidth}
\caption{\textit{Performance and Complexity Under Concept Drift.} We show the F1 scores and the log number of splits for four data sets with known concept drift. We indicate the types of concept drift in parentheses. Specifically, we show the mean and standard deviation (shaded area) for a sliding window aggregation with a window size of 20. The Dynamic Model Tree has less performance degradation and recovers faster after a concept drift, while often remaining shallower than existing models.}
\label{fig:recall_complexity}
\end{figure*}

\subsection{Performance Measures}
\subsubsection{Predictive Performance}
Classification error and accuracy are common measures for evaluating online classifiers. However, both measures might produce biased results for imbalanced data. As our evaluation incorporates many imbalanced data sets, we report the F1 measure instead. The F1 measure is the harmonic mean of precision and recall and provides reliable results even for strong imbalances.

\subsubsection{Interpretability/Complexity}\label{sec:interpretability_measure}
Since there is no common measure of interpretability, one usually resorts to heuristics. For example, one can compare the number of parameters in linear models or the number of nodes in decision trees. Unfortunately, in our case there is no clear separation between model families. In particular, a comparison between the complexities of Hoeffding and Model Trees is difficult, as their leaf nodes offer different degrees of expressiveness. Hence, we consider the \textit{number of splits} in our evaluation, which we calculated as follows: Each inner node counted as one split. Majority-weighted leaf nodes did not contribute to the total number of splits. Conversely, the leaf classifiers can be considered as another final split of the observations. Accordingly, we counted one more split for binary classifiers and $c$ more splits for multiclass classifiers, where $c$ is the number of classes. Compared to measuring the total number of nodes, the \textit{number of splits} accounts for the different leaf types of Hoeffding Trees and Model Trees. For completeness, we also report the \textit{number of parameters}. Specifically, we counted one parameter per inner node corresponding to the split value. We counted leaf nodes as either one (majority class) or $m$ (linear model weights; Na\"ive Bayes conditional probabilities) additional parameters, where $m$ is the number of features.\footnote{For multinomial classification, we counted the parameters corresponding to each class.}

In practice, it depends on a given application whether the simple leaf models should be considered as limiting interpretability or not. That is, as mentioned earlier, simple models can offer significant advantages in terms of local feature-based explainability. Accordingly, we generally consider the \textit{number of splits} to be a more reliable indication of the interpretability of incremental decision trees. Still, instead of giving too much importance to heuristics, one should aim for online learning models that have meaningful interpretability properties, such as those proposed in this paper.

\subsubsection{Computational Efficiency}
Computational efficiency depends on the respective implementation and hardware configuration. As we used both scikit-multiflow and custom implementations, we did not focus on computational efficiency in the experiments. However, for the sake of completeness, we provide the average computation time for one train/test-iteration of each model in Table \ref{tab:time}.

\begin{table}[t]
\caption{\textit{Computation Time in Seconds (lower is better).} We show the mean and standard deviation of the computation time for one test/train iteration over all data sets.}
    \label{tab:time}
    \centering
    \begin{adjustbox}{max width=\columnwidth}
        \begin{tabular}{llllll}
        \toprule
        DMT (ours) & FIMT-DD & VFDT (MC) & VFDT (NBA) & HT-ADA & EFDT\\
        \cmidrule(lr){1-6}
        0.53 ± 0.21 &   1.12 ± 0.57 &   \textbf{0.06 ± 0.03} &   0.14 ± 0.03 &  0.34 ± 0.08  &  17.23 ± 6.33 \\
        \bottomrule
        \end{tabular}
    \end{adjustbox}
\end{table}

\subsection{Results}
In the following, we discuss our most important findings.

\subsubsection{Predictive Performance}
Table \ref{tab:f1} shows the average F1 measure of all models and data sets. Using simple leaf models instead of majority voting has generally improved the obtained F1 score (see DMT, FIMT-DD and VFDT (NBA)). This advantage is most evident in the Hyperplane data set. The Hyperplane data was generated by rotating a decision hyperplane in multidimensional space. Thus, after a few splits, the observations can be linearly separated sufficiently well by the simple models. Although no model achieved good predictive quality on Poker, the VFDT with Na\"ive Bayes has a higher average score than the Model Trees, suggesting that a different simple model type may improve the results.

In general, FIMT-DD and the Dynamic Model Tree obtained similar F1 scores. However, our framework outperformed FIMT-DD for Airlines and the synthetic data sets. Looking at the behavior of FIMT-DD over time, its disadvantage can often be attributed to slow growth (e.g., Airlines, Agrawal, SEA) or aggressive pruning (e.g., Agrawal, Hyperplane). Accordingly, a less strict split threshold, different purity measures, and alternative pruning strategies could be explored in the future. Similarly, the Dynamic Model Tree may perform poorly in the first time steps if the random initial weights have not yet converged. This effect is noticeable in the averaged result of Electricity, since it is a relatively small data set. To speed up the initial training of the simple models, one may experiment with dynamic learning rates.
As can be seen from the Gas data set, the VFDT and HT-Ada implementations may have difficulty finding optimal split candidates for high-dimensional and continuous feature sets. Both models remained extremely shallow and obtained poor predictive performance. In such cases, where it is difficult or infeasible to find a good split value among all possible candidates, the simple leaf models can provide an advantage. Besides, HT-Ada was not competitive for Covertype and Insects-Abrupt. Here, pruning near the root caused temporary declines of the F1 score. Such behaviour might be avoided by a less aggressive pruning strategy or a more robust drift detection scheme.

Our framework obtained either the best or second best average F1 score for all data sets with known concept drift (TüEyeQ, Insects-Abrupt, Insects-Incremental, SEA, Agrawal and Hyperplane). We depict the detailed results of four data sets in Figure \ref{fig:recall_complexity}. The Dynamic Model Tree often suffers only minor performance deterioration after a concept drift. Compared to the other models, our framework usually recovers faster from both abrupt and incremental concept drift. The effect is most notable in the SEA and Insects data set.

In summary, the proposed Dynamic Model Tree (DMT) is among the best performing models for most data sets. In fact, our framework ranks first place on average, even when the more powerful ensemble models are taken into account.

\begin{figure}
    \centering
    \includegraphics[width=\columnwidth]{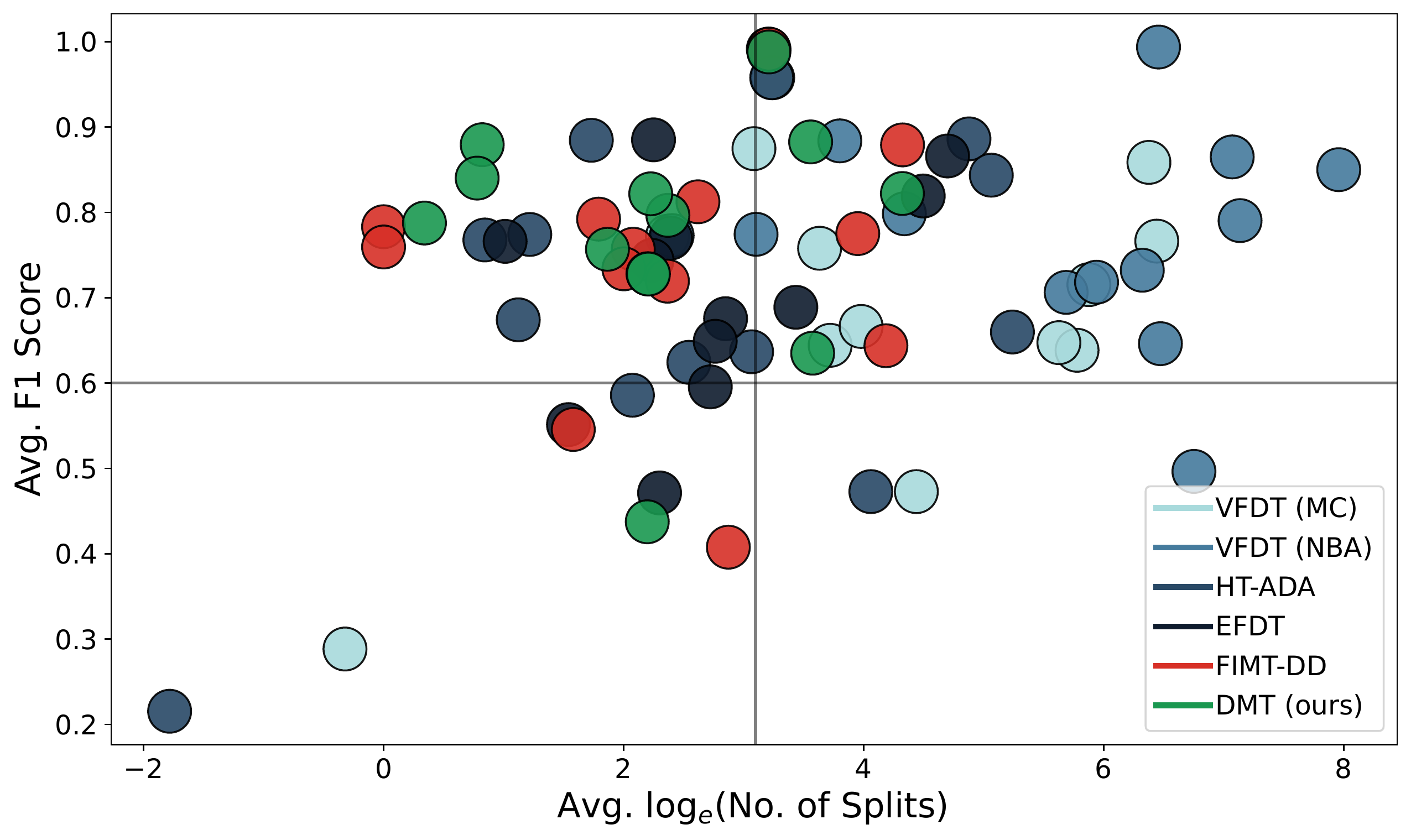}
    \caption{\textit{Predictive Performance vs. Model Complexity.} Above, we compare the F1 measure and the logarithm of the number of splits of each incremental decision tree. The number of splits is an indication of model complexity, which in turn is a common proxy for the interpretability. That is, fewer splits can usually be associated with higher interpretability. Each point corresponds to the average measure of one data set. Detailed results can be found in the Tables \ref{tab:f1} and \ref{tab:n_splits}. Ideally, we aim for a large F1 score and a small number of splits, corresponding to a value in the top left quadrant. While achieving competitive F1 scores, the Dynamic Model Tree generally manages to reduce the number of splits compared to the Hoeffding Trees.}
    \label{fig:scatter}
\end{figure}

\subsubsection{Complexity and Interpretability}
As described above, we report the \textit{number of splits} (Table \ref{tab:n_splits}) and the \textit{number of parameters} (Table \ref{tab:n_param}) as indicators for the interpretability of a model. Model Trees often maintain a shallower tree structure than Hoeffding Trees. This effect can be attributed to the additional flexibility provided by the simple models. For example, the synthetic Hyperplane and SEA data sets can both be separated by a hyperplane. The Dynamic Model Tree was able to represent these linear relationships with fewer splits than the Hoeffding Trees, while achieving similar or higher predictive quality.
The complexities of the Dynamic Model Tree and FIMT-DD often varied. We attribute this effect to the loss-based gains that allow our framework to meet the \textit{consistency with parent splits} and \textit{model minimality} properties. Specifically, while the Dynamic Model Tree will only retain a split, if it is beneficial in terms of the loss, FIMT-DD retains a split as long as the Page Hinkley test does not detect a concept drift. This may lead to overly complex trees that offer only slight or no improvements in terms of the F1 score (see Electricity and Bank). Likewise, if there is no significant difference according to the Hoeffding bound, FIMT-DD does not split a node, even though this might reduce the expected loss (see Airlines and SEA). In addition, FIMT-DD aims to reduce the standard deviation of the target and can therefore obtain leaf nodes that are extremely imbalanced towards one class. While this would be beneficial for majority weighting, it could make training simple (linear) models more difficult. Ultimately, this may reduce the predictive performance of FIMT-DD compared to a Dynamic Model Tree, even though both models have similar complexity (see Agrawal).

The Dynamic Model Tree ranks first for the average \textit{number of splits} and third for the more conservative \textit{number of parameters}. Indeed, Figure \ref{fig:recall_complexity} shows that the complexity of the Dynamic Model Tree typically remains low over time, while other methods such as VFDT produce increasingly larger trees. Besides, the Dynamic Model Tree can adapt to different types of concept drift without drastically changing its complexity.

In general, our results demonstrate that high predictive performance and low complexity need not be mutually exclusive in an evolving data stream. The relationship of predictive performance and complexity is also shown in Figure \ref{fig:scatter}. A summary of our experiments is depicted in Table \ref{tab:summary}.

\begin{table}[t]
\caption{\textit{Experiment Summary.} We provide a concise summary of our experiments. For more detailed results, please see the remaining tables and plots. We ranked all methods according to four categories. Both predictive performance categories are based on the results in Table \ref{tab:f1}. The second category reflects the average performance for the data sets with known concept drift. The complexity and efficiency scores are based on the average results in the Tables \ref{tab:n_splits} and \ref{tab:time}. We used the following methodology: The best and worst models per category have received a score of \textbf{++} and \textbf{-- --} respectively. The other methods have received a score of \textbf{+} or \textbf{--} depending on whether they were above or below the median.
}
    \label{tab:summary}
    \centering
    \begin{adjustbox}{max width=\columnwidth}
        \begin{tabular}{lcccc}
        \toprule
         & Overall & Pred. Performance  & Complexity/ & Computational \\
        Model $\backslash$ Category & Pred. Performance & For Known Drift  & Interpretability & Efficiency \\
        \cmidrule(lr){1-1} \cmidrule(lr){2-5}
        DMT (ours) & \textbf{++} & \textbf{++} & \textbf{++} & \textbf{--} \\
        FIMT-DD \cite{ikonomovska2011learning} & \textbf{+}  & \textbf{--} & \textbf{+} & \textbf{--} \\
        VFDT (MC) \cite{domingos2000mining} & \textbf{--} & \textbf{-- --} & \textbf{--} & \textbf{++} \\
        VFDT (NBA) \cite{gama2003accurate} & \textbf{+} & \textbf{+} & \textbf{-- --} & \textbf{+} \\
        HT-Ada \cite{bifet2009adaptive} & \textbf{-- --} & \textbf{--} & \textbf{--} & \textbf{+} \\
        EFDT \cite{manapragada2018extremely} & \textbf{--} & \textbf{+} & \textbf{+} & \textbf{-- --} \\
        \bottomrule
        \end{tabular}
    \end{adjustbox}
\end{table}

\section{Conclusion}\label{sec:conclusion}
In this paper, we introduced the Dynamic Model Tree, a flexible and interpretable framework for machine learning on large-scale evolving data streams. A Dynamic Model Tree adheres to sensible properties that make it a reliable choice even in highly challenging streaming scenarios. Our experiments show that the proposed framework can achieve state-of-the-art performance with a fraction of the complexity of many previous methods. In particular, the Dynamic Model Tree automatically adapts to different types of concept drift, without the need for complex model extensions common in existing frameworks. Accordingly, we hope that our work will support the current trend towards more efficient and interpretable machine learning.


\end{document}